\documentclass{article}

\usepackage{microtype}
\usepackage{graphicx}

\usepackage{subfigure}

\usepackage{booktabs, multirow} 
\usepackage{soul}
\usepackage[table]{xcolor} 
\usepackage{changepage,threeparttable} 
\usepackage{tcolorbox}

\usepackage{hyperref}

\usepackage[accepted]{icml2024}

\usepackage{amsmath}
\usepackage{amssymb}
\usepackage{mathtools}
\usepackage{amsthm}
\usepackage{listings}
\usepackage{breqn}
\usepackage{enumitem}
\usepackage{epigraph}
\usepackage[capitalize,noabbrev]{cleveref}
\usepackage{dblfloatfix}

\DeclareFontFamily{U}{mathb}{}
\DeclareFontShape{U}{mathb}{m}{n}{
  <-5.5> mathb5
  <5.5-6.5> mathb6
  <6.5-7.5> mathb7
  <7.5-8.5> mathb8
  <8.5-9.5> mathb9
  <9.5-11.5> mathb10
  <11.5-> mathb12
}{}
\DeclareSymbolFont{mathb}{U}{mathb}{m}{n}
\DeclareMathSymbol{\ulsh}{3}{mathb}{"E8}
\DeclareMathSymbol{\ursh}{3}{mathb}{"E9}
\DeclareMathSymbol{\dlsh}{3}{mathb}{"EA}
\DeclareMathSymbol{\drsh}{3}{mathb}{"EB}

\usepackage{arydshln}
\usepackage{thmtools}
\newcommand{\qs}{QuIP$\#$}
\newcommand{\qss}{\qs\ }
\newcommand{\sota}{state-of-the-art\ }

\newcommand{\norm}[1]{\left\|#1\right\|}
\newcommand{\Abs}[1]{\left|#1\right|}
\newcommand{\Exv}[1]{\mathbf{E}\left[#1\right]}
\newcommand{\Prob}[1]{\operatorname{Prob}\left(#1\right)}

\newcommand{\trace}[1]{\operatorname{tr}\left(#1\right)}
\newcommand{\R}{\mathbb{R}}

\theoremstyle{plain}
\newtheorem{theorem}{Theorem}[section]

\newtheorem{lemma}[theorem]{Lemma}

\theoremstyle{definition}
\newtheorem{definition}[theorem]{Definition}

\theoremstyle{remark}

\usepackage[textsize=tiny]{todonotes}

\icmltitlerunning{\qs}

\begin{document}

\twocolumn[
\icmltitle{\qs: Even Better LLM Quantization with \\ Hadamard Incoherence and Lattice Codebooks}



\icmlsetsymbol{equal}{*}

\begin{icmlauthorlist}
\icmlauthor{Albert Tseng}{equal,cornellcs}
\icmlauthor{Jerry Chee}{equal,cornellcs}
\icmlauthor{Qingyao Sun}{cornellorie}
\icmlauthor{Volodymyr Kuleshov}{cornellcs}
\icmlauthor{Christopher De Sa}{cornellcs}
\end{icmlauthorlist}

\icmlaffiliation{cornellcs}{Department of Computer Science, Cornell University}
\icmlaffiliation{cornellorie}{Department of Operations Research and Information Engineering, Cornell University}

\icmlcorrespondingauthor{Albert Tseng}{albert@cs.cornell.edu}
\icmlcorrespondingauthor{Jerry Chee}{jerrychee@cs.cornell.edu}

\icmlkeywords{Machine Learning, ICML, Quantization, Large Language Models, LLMs, Low Precision, Inference, Systems, Hardware, 2 bit, QuIP, Incoherence Processing, Lattice Codebooks, Vector Quantization}

\vskip 0.3in
]



\printAffiliationsAndNotice{\icmlEqualContribution}
\linepenalty=1000
\begin{abstract}
Post-training quantization (PTQ) reduces the memory footprint of LLMs by quantizing their weights to low-precision.
In this work, we introduce \qs, a weight-only PTQ method that achieves state-of-the-art results in extreme compression regimes ($\le$ 4 bits per weight) using three novel techniques.
First, \qss improves QuIP's \cite{chee2023quip} incoherence processing by using the randomized Hadamard transform, which is faster and has better theoretical properties.
Second, \qss uses vector quantization to take advantage of the ball-shaped sub-Gaussian distribution that incoherent weights possess: specifically, we introduce a set of hardware-efficient codebooks based on the highly symmetric $E_8$ lattice, which achieves the optimal 8-dimension unit ball packing.
Third, \qss uses fine-tuning to improve fidelity to the original model.
Our experiments show that \qss outperforms existing PTQ methods, enables new behaviors in PTQ scaling, and supports fast inference.
Our code can be found at \url{https://github.com/Cornell-RelaxML/quip-sharp}.
\end{abstract}

\section{Introduction}

\begin{figure}
\centering
\includegraphics[width=\linewidth]{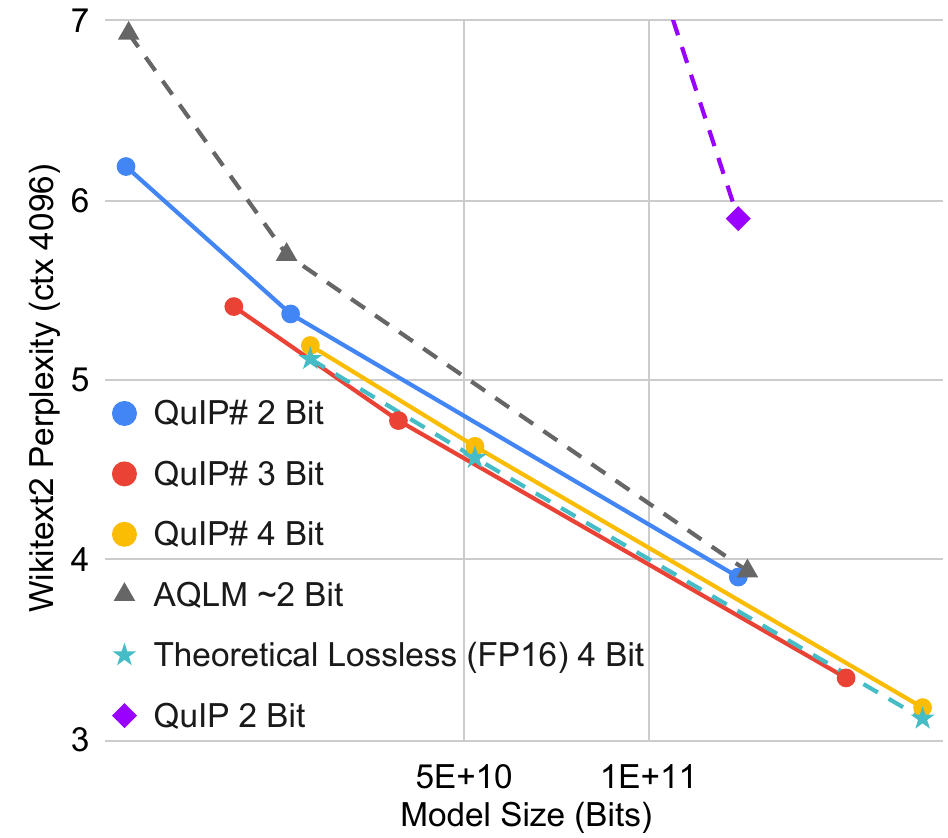}
\vspace{-0.2cm}
\caption{\qss offers unprecedented quantization quality at extreme compression ratios. \qss 3-bit models also scale better than theoretically lossless 4-bit models, a previously unseen result.}
\vspace{-0.1cm}
\label{fig:llama2scale}
\end{figure}

\begin{figure*}
\centering
\includegraphics[width=0.9\linewidth]{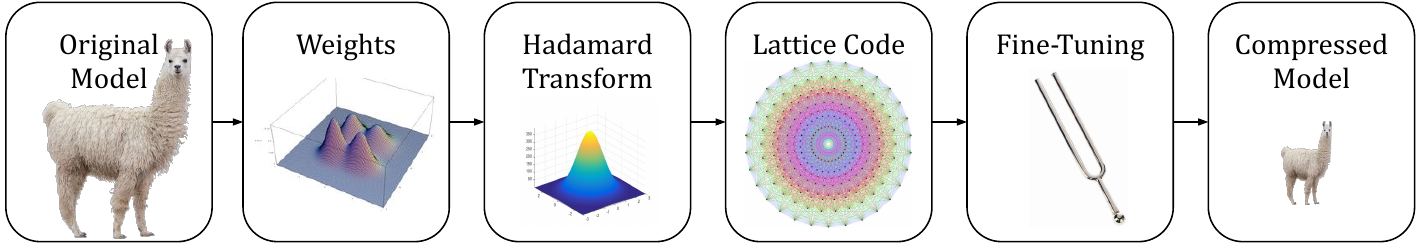}
\caption{\qss performs incoherence processing with a Randomized Hadamard Transform and uses lattice codebooks to achieve \sota quantized models.}
\end{figure*}


Large language models (LLMs) have driven rapid advances across diverse fields such as natural language processing \cite{llama2}, scientific modeling \cite{hyenadna}, and program synthesis \cite{codellama}.
However, the massive size of these models poses significant challenges to their deployment. For example, the largest model in the Llama 2 family has 70B parameters, and requires 140GB of GPU memory in native 16-bit precision \cite{llama2}. 
This massive memory footprint motivates research into lossless LLM compression methods.

Post-training quantization (PTQ) linearly reduces the memory footprint of models by storing trained weights with less precision.
For example, Llama 2 70B only requires $<20$GB of memory when quantized to 2 bits. 
This not only lets large models fit on smaller devices, but also enables faster throughput in memory bound settings such as autoregressive decoding.
However, existing quantization methods either do not scale to extreme compression ratios \cite{omniquant} or have expensive decoding schemes \cite{aqlm}, motivating the development of both good \textit{and} fast PTQ methods.

In this work, we introduce \qs, a \textit{weight-only} PTQ method that achieves a new \sota in model quantization.
\qs~improves over existing work via three techniques: incoherence processing, lattice codebooks, and fine-tuning.
Incoherence processing is a principled form of outlier suppression that produces approximately Gaussian distributed weight matrices \cite{chee2023quip}.
\qss performs incoherence processing with the computationally-efficient randomized Hadamard transform \cite{halko2011finding} (Section \ref{sec:RHT}).
To quantize incoherent matrices, \qss uses the BlockLDLQ block adaptive rounding algorithm with compressible codebooks based on the $E_8$ lattice, which achieves the highest density 8 dimensional unit-ball packing \cite{maryna} (Section \ref{sec:blockround}). 
The $E_8$ lattice is highly structured and symmetric, allowing our codebooks to be hardware-friendly and admit fast inference.
Finally, \qss includes an inter-layer fine-tuning algorithm that further improves quantization quality (Section \ref{sec:ftorig}).

\qss significantly outperforms existing PTQ methods including OmniQuant \cite{omniquant}, QuIP \cite{chee2023quip} (a previous, separate work), and AQLM \cite{aqlm}.
To the best of our knowledge, \qss is also the first PTQ method where 3-bit models scale better than 4-bit models.
This directly refutes \citet{dettmers2022scaling}'s claim that 4-bit models are ``optimal'' and indicates that as the field of PTQ develops, 2-bit models are likely to scale better than 3-bit models in the near future.
Moreover, \qss was designed from the ground up to be fast.
Algorithm \ref{alg:qsinf} describes fast inference with a \qs-quantized linear layer.
Our ``proof of concept'' CUDA implementation of \qss achieves over 50\% of peak memory bandwidth on a NVIDIA RTX 4090, validating our design choices.

\textbf{In summary}, we introduce \qs, a post-training quantization method that achieves \sota results by
\begin{enumerate}[nosep]
\item Performing incoherence processing with the Randomized Hadamard Transform, which has better incoherence properties and faster runtime than the Kronecker factorization in QuIP.
\item Rounding incoherence-processed weight matrices with block adaptive rounding and codebooks based on the $E_8$ lattice, which achieves the highest 8-dimension unit ball packing density (kissing number).
\item Introducing an inter-layer fine-tuning algorithm that further improves quantization quality.
\end{enumerate}

\begin{algorithm}
\caption{\qss without Fine-Tuning (\qs-NoFT)}
\begin{algorithmic}
\INPUT Weight $W \in \mathbb{R}^{m \times n}$, hessians $H\in\mathbb{R}^{n\times n}$, $g$-dim. $k$-bit codebook $C$
\STATE $\hat W, \hat H, S_U, S_V \gets \mbox{IP-RHT}(W, H)$ (Alg. \ref{alg:RHT})
\STATE $\hat W \gets \mbox{BlockLDLQ}(\hat W, \hat H, C)$ (Sec. \ref{sec:blockldlq})
\OUTPUT $\hat W, S_U, S_V$
\end{algorithmic}
\label{alg:qsnoft}
\end{algorithm}

\begin{algorithm}
\caption{\qss Inference (for a Linear Layer)}
\begin{algorithmic}
\INPUT $\hat W$, $S_U, S_V$ from Alg. \ref{alg:qsnoft}, $g$-dim. $k$-bit codebook $C$, input $x \in \mathbb{R}^n$.
\STATE $y \gets \texttt{Had}(S_V \odot x)$ where $\texttt{Had}$ performs \\ \hspace{2em} an orthogonal Hadamard transform (Sec. \ref{sec:RHT})
\STATE $y \gets \texttt{decompress_multiply}(\hat W, C, y)$
\STATE $y \gets \texttt{Had}(S_U \odot y)$
\OUTPUT $y$ 
\end{algorithmic}
\label{alg:qsinf}
\end{algorithm}

\section{Background / Related Work}

\subsection{Compressing LLMs}

A large body of work has focused on compressing LLMs, as doing so can directly benefit LLM inference at scale.
Methods such as pruning, quantization aware training (QAT), and post-training quantization (PTQ) all focus on different areas of this problem and are not strictly orthogonal to each other.
Pruning \textit{removes} weights from models while preserving model quality and inference performance \cite{jerrypruning, wanda}.
QAT focuses on training models that are more ``quantizable'' but usually requires training models from scratch \cite{nagelqat,xi20234bit}. 
PTQ, which \qss falls under, instead quantizes \textit{pre-trained} models.
PTQ requires less compute than QAT and achieves competitive performance \cite{chee2023quip,optq,omniquant,aqlm}.
For the rest of this paper, we focus on PTQ.

\subsection{Quantization and Adaptive Rounding}
\label{sec:quantbg}
In \qs, we follow existing \sota PTQ methods and round weights to minimize the per-layer proxy loss, as formalized by \citet{nagelround}:
\begin{align}
\label{eq:round_proxy}
\ell(\hat W) &= E_x \left[ \| (\hat W - W)x \|^2 \right] \\
  &= \operatorname{tr}\left((\hat W - W) H (\hat W - W)^T\right) .
\label{eqn:proxy}
\end{align}
Here, $W \in \mathbb{R}^{m \times n}$ is the original weight matrix in a linear layer, $\hat W \in \mathbb{R}^{m \times n}$ are the quantized weights, $x \in \mathbb{R}^n$ is an input vector drawn uniformly at random from a calibration set, and $H = E_x[xx^T]$ is a proxy Hessian.
This intra-layer formulation makes quantization tractable for LLMs.
One way to minimize $\ell$ is to use adaptive rounding methods that iteratively round weight matrices by considering the current rounding error for that specific matrix.
For example, the LDLQ\footnote{OPTQ \cite{optq} and QuIP independently introduced alternative formulations of this rounding method, and QuIP showed them to be equivalent. LDLQ is the name given by QuIP.} rounding algorithm iteratively rounds rows of model weights using linear feedback from quantization error of already rounded rows.
LDLQ is optimal within the class of adaptive rounding methods with linear feedback and offers provably better error rates than nearest or stochastic rounding \cite{chee2023quip}.



\subsection{Incoherence Processing}
\label{sec:incpbg}

Multiple works have observed that outliers in model activations and weights can hinder quantization quality, motivating methods that ``suppress'' outliers during quantization.
For example, AWQ \cite{awq} scales model weights by information from activations and OmniQuant \cite{omniquant} uses simple learnable model-preserving transformations.
However, these heuristic-based approaches tend to fail at lower bitrates.

Instead, in QuIP, \citet{chee2023quip} proposed that \emph{incoherence} is important for LLM quantization.
Informally, incoherent matrices have concentrated entry magnitudes---ruling out outliers.
In LLMs, incoherent weight and Hessian matrices mean that both the thing being rounded (weights) and important rounding directions (Hessians) are not too large in any coordinate.
This enables quantization with \emph{provably} bounded error.
\begin{definition}[\citet{chee2023quip}]
A Hessian $H \in \mathbb{R}^{n \times n}$ is $\mu$-incoherent if its eigendecomposition $H = Q \Lambda Q^T$ has \[\textstyle \max_{i,j} \; |Q_{ij}| = \max_{i,j} \; |e_i^T Q e_j| \leq \mu / \sqrt{n}.\] 
A weight matrix $W \in \mathbb{R}^{m \times n}$ is $\mu$-incoherent if \[\max_{i,j} \; \textstyle |W_{ij}| = \max_{i,j} \; |e_i^T W e_j| \leq \mu \|W\|_F / \sqrt{mn}.\]
\end{definition}
To exploit incoherence, \citet{chee2023quip} introduced \emph{incoherence processing} as a part of their quantization method QuIP.
QuIP's incoherence processing works by conjugating $W$ and $H$ by structured random orthogonal matrices.
Specifically, QuIP constructs orthogonal matrices $U \in \mathbb{R}^{m \times m}$ and $V \in \mathbb{R}^{n \times n}$ via a Kronecker product by drawing uniform random orthogonal matrices $U_1$, $U_2$ (of sizes about $\sqrt{n}$), $V_1$, and $V_2$ (of sizes about $\sqrt{m}$) and setting $U = U_1 \otimes U_2$ and $V = V_1 \otimes V_2$.
If we assign $\tilde H \gets V H V^T$ and $\tilde W \gets U W V^T$, $\tilde H$ and $\tilde W$ become $\tilde{O}(1)$-incoherent with high probability (see their Lemma~5).
Note that this transformation preserves the proxy objective, as $\trace{(U W V^T) (V H V^T) (V W^T U^T)} = \trace{WHW^T}$.
After quantizing the transformed weight matrix $\tilde W$ using $\tilde H$, during inference, QuIP-quantized models transform model activations $x$ with $V$ and $U^T$ to compute \[
\textstyle U^T (\operatorname{quantized}(\tilde W) (V x)) \approx U^T (\tilde W (V x)) = Wx. \]
These structured orthogonal multiplies by a Kronecker product lead to a runtime overhead of $\Theta(n \sqrt{n} + m \sqrt{m})$, which is small relative to the $\Theta(mn)$ cost of the multiply by $W$.

Incoherence processing can be seen as a principled alternative to more complicated and heuristic methods for outlier suppression.
Methods such as grouping and keeping outliers in FP16 require extra storage and can negatively impact performance.
For example, using a 16 bit scale per group of 64 weights requires an extra 0.25 bits per weight.
This increase is significant in extreme compression regimes, whereas incoherence processing has minimal inference overhead and allows more bits to be spent on actually quantizing model weights.
Alternatively, keeping outliers in high precision requires storing unstructured high precision matrices, which are slow to multiply by.

\subsection{Vector Quantization}
\label{sec:vq}
Prior PTQ works have focused on quantizing each scalar weight $W_{ij}$ individually, amounting to scalar quantization (SQ) \cite{chee2023quip, awq, omniquant}. 
However, SQ is subotimal as it ignores the shape of the source distribution.
Vector quantization (VQ) instead quantizes a group of $d$ weights together as a $d$ dimensional vector.
In $k$-bit VQ, a vector is quantized to one of $2^{kd}$ vectors $\in \mathbb{R}^d$ that form a $2^{kd} \times d$ codebook $C$. 
By shaping $C$ to the source distribution of $W$, VQ can achieve lower distortion than SQ, with higher $d$ enabling better shaping \cite{kostina}.

However, VQ has exponential cost in both the bitrate and vector dimension. 
As such, VQ can be expensive and can have limited distortion gains over SQ due to practical constraints on $d$.
For example, for fast inference on GPUs, $C$ must fit in L1 cache even after bank conflicts (32$\times$ duplication).
This means that $kd$ can be at most $\approx10$ for an unstructured $C$.
In \qs, we mitigate these issues by using a highly structured 2-bit codebook based on the 8D $E_8$ lattice, E8P.
E8P achieves $kd=16$ but can be compressed $256\times$, allowing it to fit in GPU cache.


\subsection{Fine-Tuning vs. Quantization Aware Training}
\label{sec:FTvsQAT}

Fine-tuning (FT) for LLM PTQ was introduced in AQLM \cite{aqlm} as a tractable way to capture inter-layer interactions.
As presented in AQLM and here, fine-tuning is essentially a hybrid method between pure PTQ and full QAT that requires significantly less data and compute than full QAT.
With \qs, fine-tuning generally matches the performance of QAT, with the caveat that QAT for LLMs is a relatively underexplored area. 
For example, with some extrapolation, LLM-QAT \cite{liu2023llmqat} 4 bit (4-16-16) performs around the same as \qss 4 bit with or without FT. 
However, \qss can quantize a 70B parameter model in a few hours on a single 8 GPU node while LLM-QAT needs 960 GPU-hours to generate training data alone.
Since fine-tuning for PTQ is a very recent development, both the methods presented here and in AQLM are almost certainly not optimal.
However, they serve to show that FT is a relatively cheap way to achieve QAT-quality models, making such an approach practical and promising. 
\section{Incoherence Processing with the Randomized Hadamard Transform}
\label{sec:RHT}

\begin{algorithm}[t]
\caption{Incoherence Processing with RHT (IP-RHT)}
\begin{algorithmic}
\INPUT $W \in \mathbb{R}^{m \times n}, H \in \mathbb{R}^{n \times n}$
\STATE Sample sign vectors $S_V \sim \mathcal{U}\{\pm1\}^n, S_U \sim \mathcal{U}\{\pm1\}^m$
\STATE $\hat W \gets \texttt{Had}(diag(S_U)\texttt{Had}(diag(S_V) W^T)^T)$ where 
\\ \hspace{2.5em} $\texttt{Had}$ is the Hadamard transform (sec. \ref{sec:RHT})
\STATE $\hat H \gets \texttt{Had}(diag(S_V) \texttt{Had}(diag(S_V) H)^T)$
\OUTPUT $\hat W, \hat H, S_U, S_V$
\end{algorithmic}
\label{alg:RHT}
\end{algorithm}

In this section, we propose a way of improving the incoherence processing of QuIP by replacing the 2-factor Kronecker product by a Randomized Hadamard Transformation (RHT)~\cite{halko2011finding}.
This change yields three advantages: (1) the theoretical bound on the incoherence parameter $\mu$ is improved; (2) the asymptotic cost of multiplying by the structured random orthogonal matrix is improved from $\Theta(n \sqrt{n})$ to $\Theta(n \log n)$; (3) the cost to multiply is further reduced by a constant factor, since a Hadamard matrix multiply can be performed without any floating-point multiplies as its entries are in $\{-1,+1\}$.
Additionally, we show in Section~\ref{sec:ablations} that this change by itself improves the perplexity of quantized LLMs.

Recall from section \ref{sec:incpbg} that one way to efficiently perform incoherence processing is to conjugate $W$ and $H$ by structured random orthogonal matrices.
\qss uses the RHT, which performs $x \to VSx$ where $V \in \mathbb{R}^{n \times n}$ is a Hadamard matrix, $S$ is a random sign vector $\{\pm 1\}^n$, and $x \in \mathbb{R}^n$.
The RHT can be computed in $O(n \log n)$ time with the Fast Walsh-Hadamard Transform \cite{fwht} when $n$ is a power of $2$.
We will temporarily assume that all dimensions are powers of 2.
Later in the section we will explain 2 methods for incoherence processing when the dimension is not a power of 2.

\begin{restatable}{lemma}{lemmahadincoh}
\label{lemma:had_incoh}
Let $H$ be any positive semidefinite matrix on $\mathbb{R}^{n \times n}$ and $W$ any weight matrix on $\mathbb{R}^{m \times n}$.
Let $U \in \mathbb{R}^{m \times m}$ and $V \in \mathbb{R}^{n \times n}$ be orthogonal scaled Hadamard matrices. Let $S_U \in \mathbb{R}^{m \times m}$ and $S_V \in \mathbb{R}^{n \times n}$ be random diagonal matrices with independent diagonal elements drawn uniformly from $\{-1,+1\}$.
Then for any $\delta > 0$, $V S_V H S_V V^T$ is $\mu_H$-incoherent with probability at least $1 - \delta$, and $U S_U W S_V V^T$ is $\mu_W$-incoherent with probability at least $1 - \delta$, where
\[
\mu_H = \sqrt{2 \log\left( \frac{2 n^2}{\delta} \right)}
\;\;\text{  and  }\;\;
\mu_W = 2 \log\left( \frac{4 mn}{\delta} \right).
\]
\end{restatable}

In QuIP \cite{chee2023quip}, the 2-factor Kronecker approach achieves $\mu_W^{Kron} = A^2 \log\left( 4 C m n / \delta \right)^2$, where $A$ and $C$ are global constants independent of $n$ and the number of factors.
\qs's RHT achieves superior incoherence via a log dependence on the matrix size rather that the Kronecker method's log-squared dependence.
All of QuIP's theory analyzing the proxy loss in Eq.~\eqref{eq:round_proxy} still holds with the RHT, with the improved incoherence rates propagating through.

Now, what about dimensions $n$ that are not powers of 2?
In most cases, we can factorize $n = pq$ where that $p$ is the largest power of 2 such that there exists a known Hadamard matrix of size $q$.
This allows us to construct $V \in \mathbb{R}^{n \times n} = H_p \otimes H_q$ where $H_p$ and $H_q$ are size $p$ and $q$ Hadamard matrices, respectively.
Then we can compute $VSx$ in $O(q^2 p \log p)$ time, which is faster than the 
$O(n(p+q))$ time of QuIP's 2-factor Kronecker approach when $p \gg q$.
For example, Llama 2 70B has intermediate dimension $28672 = 1024 * 28$; $1024 \gg 28$.
Algorithm \ref{alg:RHT} describes how to perform incoherence processing with the RHT.
Doing so requires storing two sign vectors $S_U \in \{\pm 1\}^m$ and $S_V \in \{\pm 1\}^n$.
Since $n, m \gg 1000$ for LLMs, $S_U$ and $S_V$ add less than 0.01 bits per weight (see Section \ref{sec:accounting} for more details).

While the Hadamard conjecture states that $\exists H_k \forall k, 4 \mid k$, finding such Hadamard matrices is still an open problem \cite{hadconj}.
In cases when there does not exist a factorization $n=pq$ where $\exists H_p, H_q$, we present a Randomized Fast Fourier Transform (RFFT) incoherence processing algorithm with similar runtime and concentration properties as the RHT.
At a high level, the RFFT performs incoherence processing with the Fast Fourier Transform (FFT) \cite{fft} and a random complex phase.
The RFFT only requires $n$ to be even, which is much weaker than the RHT's restrictions on $n$.
The RFFT is also useful when there \textit{does exist} a decomposition $n=pq$ but $p \not \gg q$, resulting in reduced speedups over an $\Theta(n\sqrt{n})$ algorithm. 
The FFT itself is also well supported on a wide variety of hardware, meaning that it may be easier to implement a fast RFFT when adapting \qss to new hardware.
In practice, we find that the RFFT performs slightly worse than the RHT but still achieves strong results (Table \ref{tab:ablate_incoherence}).
We describe the RFFT in detail in Section \ref{sec:fft} in the Appendix.

\begin{table}[ht]
\centering
\caption{
RHT vs. RFFT incoherence processing using 2 Bit \qss (no FT).
WikiText2 perplexity ($\downarrow$), context length 4096.}
\label{tab:ablate_incoherence}
\small\sc
\begin{tabular}{@{}cccc@{}}
\toprule
Incoherence & 2-7B & 2-13B & 2-70B \\ \midrule
Hadamard & 8.22 & 6.06 & 4.16 \\
Fourier & 8.30 & 6.08 & 4.17 \\ \bottomrule
\end{tabular}
\end{table}

\section{BlockLDLQ and Lattice Codebooks}
\label{sec:blockround}
It follows from the central limit theorem that RHT-transformed weights follow a roughly ball-shaped Gaussian distribution.
However, rounding weights one at a time, as QuIP does with its LDLQ, ignores this shaping---producing a set of representable weight vectors that is shaped like a hypercube rather than a ball.
Vector quantization (VQ) lets us shape codebooks to better match the source distribution.
VQ codebooks quantize multiple weights to a single codebook entry, and we design the overall shape of our codebook to better match the roughly ball shape of the RHT transformed weights.
In Section \ref{sec:blockldlq}, we introduce BlockLDLQ, which adaptively rounds blocks of weights with VQ.
Within BlockLDLQ's VQ step, \qss uses the 2 bit E8P codebook (Section \ref{sec:e8pconstruct}).
E8P is based on the $E_8$ lattice, which achieves the highest density unit ball packing in $\mathbb{R}^8$ \cite{maryna}.
E8P achieves good shaping while enabling fast inference by only needing to look up from a $256\times 8$ codebook. 

\subsection{Adaptive Rounding for Vector Quantization}
\label{sec:blockldlq}

\citet{chee2023quip} formulated a class of adaptive rounding algorithms with linear feedback.
These methods round columns one at a time with linear feedback $a_k$ from the already rounded columns.
Specifically, columns of a weight matrix $W \in \mathbb{R}^{m \times n}$ are iteratively rounded for $k = 1, 2, \dots, n$:
$\hat W_k = \mathcal{Q}(W_k + (W_{:(k-1)} - \hat W_{:(k-1)})a_k),$
where $W_k$ is the $k$-th column of $W$, $W_{:(k-1)}$ is the first $k-1$ columns of $W$, $\mathcal{Q}$ performs nearest or stochastic rounding, and $a_k \in \mathbb{R}^{k-1}$.
The resulting $\hat W$ satisfies $\hat W = \mathcal{Q}(W + (W - \hat W) U)$, where $U \in \mathbb{R}^{n \times n}$ is a upper triangular matrix whose columns are $a_k$ and $\mathcal{Q}$ acts elementwise.

The LDLQ algorithm sets U to be $L^T - I$ where $H = L^TDL$ is the LDL decomposition of the proxy Hessian $H$.
From QuIP, we know that LDLQ is optimal within adaptive rounding methods with linear feedback when rounding to the integers.
However, LDLQ does not work with vector quantization, which rounds multiple columns together.
Here, we extend LDLQ to support vector quantization. Given a block size $g$ that evenly divides $n$, our block LDLQ is based on a novel $g$-block LDL decomposition $H = \mathbf{L}^T \mathbf{D} \mathbf{L}$, where $\mathbf{L}$ is a unit block lower triangular matrix (among the $n^2/g^2$ $g \times g$ blocks of $L \in \mathbf{R}^{n \times n}$, the $n/g$ diagonal blocks are all $I$ and all blocks above the diagonal are $0$), and $\mathbf{D}$ is a block diagonal matrix.\footnote{It is straightforward to produce the $g$-block LDL decomposition from the Cholesky decomposition of $H$.}
As before, we set $\mathbf{U} = \mathbf{L}^T - I$, and round $W$ in a block-wise fashion via
\[\hat W_k = \mathbf{Q}(W_k + (W_{:(k-1)} - \hat W_{:(k-1)})\mathbf{A}_k),\]
where $\mathbf{A}_k \in \mathbb{R}^{n \times g}$ contains the $k - g + 1$ through $k$-th columns of $\mathbf{U}$ (the $k$th block), $W_k$ similarly denotes the $k$th block of $W$, and $\mathbf{Q}$ denotes a vector quantizer.
As in the original QuIP paper, we can bound the error of this method.
\begin{restatable}{theorem}{thmLDLQ}
\label{theorem:LDLQ}
Suppose that we round $W \in \mathbb{R}^{m \times n}$ using $g$-block LDLQ with Hessian $H$, producing $\hat W$. Suppose that $H$ is $\mu$-incoherent, and that we use a (possibly stochastic) vector quantizer $\mathbf{Q}$ that satisfies $\mathbf{E}[(\mathbf{Q}(x) - x)(\mathbf{Q}(x) - x)^T] \preceq \sigma^2 I$ for any $x \in \mathbb{R}^g$. Then
\[
\textstyle
\mathbf{E}[\operatorname{tr}((\hat W - W)H(\hat W - W)^T)]
\le
\frac{g m \mu^2 \sigma^2}{n} \operatorname{tr}(H^{1/2})^2.
\]
\end{restatable}
Observe that under the same conditions, just quantizing all blocks independently would yield $\mathbf{E}[\operatorname{tr}((\hat W - W)H(\hat W - W)^T)]
\le
g m \sigma^2 \operatorname{tr}(H)$: this ``improvement'' from the trace of $H$ to the square of the trace of its square root divided by $n$ is the same factor achieved in the scalar case in QuIP.\footnote{The original QuIP paper also included multiple other technical guarantees, including a bound that considers more rigorously the ``real'' case of finite-sized codebooks. While these results could also be generalized to the block-LDLQ case, we view this as not providing much insight relevant to \qss{} beyond Theorem~\ref{theorem:LDLQ}, so (if desired) they are left as an exercise for the reader.}



\subsection{The E8P (``E8 Padded'') Codebook}



\label{sec:e8pconstruct}

BlockLDLQ relies on an internal vector quantization (VQ) step $\mathbf{Q}$ that rounds a $d$-dimension ($g$ in the previous section) vector to a codebook $C$.
To effectively apply VQ, $C$ should be shaped like the source distribution and have high packing density.
One way to improve shaping is by increasing $d$. 
However, recall from Section \ref{sec:vq} that to quantize a vector $v \in \mathbb{R}^{d}$ to $k$ bits with VQ, $C$ must have size $2^{kd} \times d$.
Since the codebook size is exponential in both the vector dimension and bitrate, VQ quickly becomes intractable at high dimensions or bitrates.

In \qs, we introduce the novel 2-bit 8 dimensional \emph{E8P codebook}, which contains $2^{16}$ entries but only requires lookups into a $2^8$-entry table, with the remaining $8$ bits being used to store signs and shifts.
E8P requires only 1KiB of space and therefore fits in the L1 cache of any modern GPU, even after duplicating for bank conflicts ($32 \times$).
E8P mitigates the scaling issues of VQ by taking advantage of the structure and symmetries of the $E_8$ lattice on which it is based.
The $E_8$ lattice is composed of all-integer or all-half-integer vectors in $\mathbb{R}^8$ whose sum is an even number, that is
\[ 
    \textstyle
    E_8 = \left(\mathbb{Z}^8 \cup \left(\mathbb{Z}^8+\frac{1}{2}\right) \right)
    \cap
    \left\{ x \mid \mathbf{1}^T x \text{ is even} \right\}.
\]
The construction of the E8P codebook starts with an equivalent way to write $E_8$ via the $\hat D_8$ lattice, where $\hat D_8 = \left\{ x \in \mathbb{Z}^8 + \frac{1}{2} \mid \mathbf{1}^T x \text{ is even} \right\}$ is the set of half-integer vectors with even parity: here, $E_8 = \hat D_8 \cup (\hat D_8 + \frac{1}{2})$. It follows that $(\hat D_8 - \frac{1}{4}) \cup (\hat D_8 + \frac{1}{4}) = E_8 + \frac{1}{4}$ is just a shifted copy of $E_8$ (keeping the same optimal packing density).

$\hat D_8$ has nice symmetry properties: flipping any (nonzero) even number of signs of an element in $\hat{D}_8$, yields another distinct element in $\hat{D}_8$. This means that if $|\hat D_8|$ denotes the set of elementwise absolute values of entries in $\hat D_8$, then each element of $\hat D_8$ can be expressed (uniquely) as the elementwise product of an entry $s \in |\hat D_8|$ and a sign vector of appropriate parity.
So, if we start from some ``source codebook'' of absolute entries $S \subset |\hat{D_8}|$, we can use the 128 possible odd- or even-parity sign flips to generate a subset of $\hat{D_8}$.
Each entry in $S$ is either an odd or even number of flips away from an entry in $\hat{D_8}$, but not both.
Thus, given $s \in S$ and 7 out of the 8 sign flips, we can infer the last one from the parity of the 7 sign flips and $s$.
This lets us use the following pattern to store a 16-bit codeword in $E_8 + \frac{1}{4}$: 8 bits for the entry in $S$, 7 bits for sign flips, and 1 bit to $\pm\frac{1}{4}$.
This lets us decode a size $2^{16}$ codebook by looking up into only a size $2^8$ codebook ($S$) and performing some operations.
All that remains is how to choose $S$:
we set $S$ to be the 227 elements of $|\hat{D_8}|$ with norm $\le \sqrt{10}$ plus 29 ``padding'' elements from $|\hat{D_8}|$ with norm $\sqrt{12}$ (see Section \ref{sec:e8ps}).
We call this ball-shaped $2^{16}$-entry lattice codebook ``E8P.''

\begin{figure}
\centering
\includegraphics[width=\linewidth]{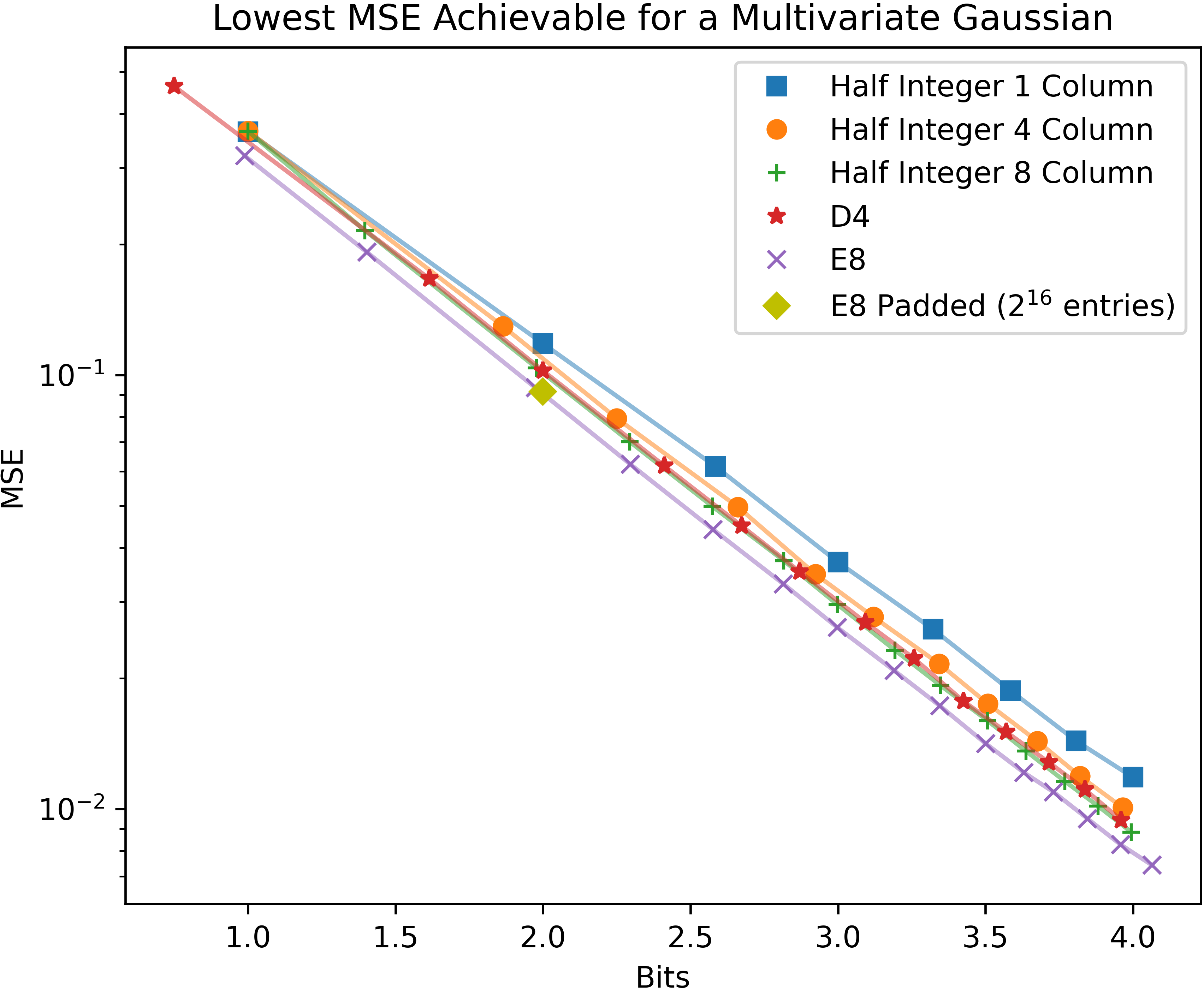}
\vspace{-0.25in}
\caption{Minimum achievable elementwise MSE of quantizing a Gaussian to various codebooks. $E_8$-based codebooks outperform other presented codebooks due to the underlying packing density and high dimensionality of $E_8$.
}
\label{fig:latticemse}
\vspace{-0.25in}
\end{figure}

Figure \ref{fig:latticemse} plots the elementwise MSE of quantizing a standard multivariate Gaussian to various $k$ bit codebooks.
Each $k$-bit codebook consists of a $d$-dimensional base lattice intersected with a ball to reach $2^{kd}$ points.
The $E_8$-based codebooks achieve lower MSEs than all other presented codebooks, including those based on the $D_4$ lattice (the even-parity vectors in $\mathbb{Z}^4$), which achieves the kissing number in $\mathbb{R}^4$. 
This figure illustrates the importance of dimension for vector quantization.
Increasing the vector dimension decreases the error for the half integer grid, as the resulting codebook is closer in shape to the source distribution.
Finally, while K-means on the source distribution would achieve lower MSE \cite{lloyd}, there are a number of practical reasons why a K-means based codebook would be less practical, including worse end-to-end empirical performance.
We discuss this more in Section \ref{sec:kmeans}.


\subsection{Scaling $E_8$ to Higher Bitrates}

The $E_8$ lattice works well for low bitrates (e.g. 2 bits), but quickly becomes intractable at higher bitrates due to codebook size.
In \qs, we use residual vector quantization (RVQ) \cite{rvq} to get the benefits of lattice codebooks at higher bitrates.
RVQ quantizes a vector $x$ to $p$ bits with a set $q$ of $q_i$-bit codebooks (denoted $\mbox{RVQ}(x, p, q)$ where $p = \sum_{0 \le i < |q|} q_i$) by repeatedly quantizing the quantization residual.
That is, $\mbox{RVQ}(x, p, q) = \sum_{0 \le i < |q|} \delta_i$ where $\delta_i = Q_{q_i}\left((x - \sum_{0 \le j < i}\delta_j) / s_i \right) \cdot s_i$, we let $Q_{q_i}(\cdot)$ denote quantizing to a $q_i$ bit codebook, and $s_i \in \mathbb{R}$.
Using RVQ, we can quantize to 4 bits by rounding with the 2 bit E8P codebook twice.
We can also quantize to 3 bits by using the 2 bit E8P codebook and a 1-bit $E_8$ codebook (elements of $E_8$ with norm $\le 2$ and 15 elements of $E_8$ with norm 4).
One could also use more advanced multi-codebook quantization approaches other than RVQ, but we found that RVQ was sufficient to achieve strong quantization performance.

\section{Fine-Tuning During Quantization}
\label{sec:ftorig}
Recent works have suggested that inter-layer interactions are important for lossless extreme quantization \cite{omniquant, aqlm}.
Here, we employ a simple fine-tuning algorithm that attempts to recover the original unquantized model during quantization. 
Our fine tuning method runs on a small development set and can be performed in around 50 GPU-hours for a 70B parameter model.

First, we fine-tune within each transformer block by fine-tuning unquantized layers to compensate for \textit{already-quantized} layers before quantization. 
This mitigates the activation error caused by an individual linear layer \textit{during quantization}, and can be parallelized across transformer blocks.
The idea of fine-tuning within a transformer block was previously proposed in \citet{aqlm}; our methodology differs in how we fine tune (before quantization) and the set of tunable parameters.
Second, after all linear layers in the model are quantized, the remaining unquantized parameters are fine-tuned to minimize activation error over the \textit{entire model}.
By optimizing the sign vectors as real vectors instead of binary vectors in both steps, we allow the incoherence processing step to shape the weight matrix \textit{to the codebook}.
While this means we must store the sign vectors in FP16 instead of as bitvectors, the size of LLM matrices means that the sign vectors still add less than 0.01 bits per weight.
We describe these steps in more detail in Section \ref{sec:ft}.

%
\section{Experiments}


\begin{table*}[ht]
\caption{Llama 1 \& 2 Wikitext2 and C4 perplexity ($\downarrow$), context length 2048.} 
\label{tab:ppl2048}
\centering
\small\sc
\tabcolsep=0.15cm
\begin{tabular}{@{}cccccccccccccccc@{}}
\multicolumn{1}{l}{} & \multicolumn{1}{l}{}                           & \multicolumn{7}{c}{Wikitext 2}                                                                                                                             & \multicolumn{7}{c}{C4}                                                                                        \\ \midrule
Method               & \multicolumn{1}{c|}{Bits}                      & 1-7           & 1-13          & 1-30          & 1-65          & 2-7           & 2-13          & \multicolumn{1}{c|}{2-70}                                  & 1-7           & 1-13          & 1-30          & 1-65          & 2-7           & 2-13          & 2-70          \\ \midrule
FP16                 & \multicolumn{1}{c|}{16}                        & 5.68          & 5.09          & 4.10          & 3.53          & 5.47          & 4.88          & \multicolumn{1}{c|}{3.32}                                  & 7.08          & 6.61          & 5.98          & 5.62          & 6.97          & 6.47          & 5.52          \\
\rowcolor[HTML]{D9D9D9} 
AWQ                  & \multicolumn{1}{c|}{\cellcolor[HTML]{D9D9D9}4} & 6.08          & 5.34          & 4.39          & 3.76          & 6.15          & 5.12          & \multicolumn{1}{c|}{\cellcolor[HTML]{D9D9D9}-}             & 7.52          & 6.86          & 6.17          & 5.77          & 7.68          & 6.74          & -             \\
\rowcolor[HTML]{D9D9D9} 
OmniQ                & \multicolumn{1}{c|}{\cellcolor[HTML]{D9D9D9}4} & 5.86          & 5.21          & 4.25          & 3.71          & 5.74          & 5.02          & \multicolumn{1}{c|}{\cellcolor[HTML]{D9D9D9}3.47}          & 7.34          & 6.76          & 6.11          & 5.73          & 7.35          & 6.65          & 5.65          \\
\rowcolor[HTML]{D9D9D9} 
QuIP\# \tiny{no FT \& no $E_8$}                & \multicolumn{1}{c|}{\cellcolor[HTML]{D9D9D9}4} & 5.83          & 5.20          & 4.23          & 3.63          & 5.66          & 5.00          & \multicolumn{1}{c|}{\cellcolor[HTML]{D9D9D9}3.42}          & 7.25          & 6.70          & 6.06          & 5.68          & 7.17          & 6.59          & 5.59          \\
\rowcolor[HTML]{D9D9D9} 
QuIP\#               & \multicolumn{1}{c|}{\cellcolor[HTML]{D9D9D9}4} & \textbf{5.76} & \textbf{5.17} & \textbf{4.18} & \textbf{3.60} & \textbf{5.56} & \textbf{4.95} & \multicolumn{1}{c|}{\cellcolor[HTML]{D9D9D9}\textbf{3.38}} & \textbf{7.18} & \textbf{6.67} & \textbf{6.03} & \textbf{5.66} & \textbf{7.07} & \textbf{6.54} & \textbf{5.56} \\
AWQ                  & \multicolumn{1}{c|}{3}                         & 11.9          & 7.45          & 10.0          & 5.21          & 24.0          & 10.5          & \multicolumn{1}{c|}{-}                                     & 13.3          & 9.13          & 12.7          & 7.11          & 23.9          & 13.1          & -             \\
OmniQ                & \multicolumn{1}{c|}{3}                         & 6.49          & 5.68          & 4.74          & 4.04          & 6.58          & 5.58          & \multicolumn{1}{c|}{3.92}                                  & 8.19          & 7.32          & 6.57          & 6.07          & 8.65          & 7.44          & 6.06          \\
QuIP\# \tiny{no FT \& no $E_8$}                & \multicolumn{1}{c|}{3}                         & 6.29          & 5.52          & 4.54          & 3.91          & 6.19          & 5.34          & \multicolumn{1}{c|}{3.71}                                  & 7.82          & 6.98          & 6.29          & 5.86          & 7.85          & 6.98          & 5.78          \\
QuIP\#               & \multicolumn{1}{c|}{3}                         & \textbf{5.98} & \textbf{5.31} & \textbf{4.36} & \textbf{3.78} & \textbf{5.79} & \textbf{5.10} & \multicolumn{1}{c|}{\textbf{3.56}}                         & \textbf{7.39} & \textbf{6.83} & \textbf{6.17} & \textbf{5.77} & \textbf{7.32} & \textbf{6.72} & \textbf{5.67} \\
\rowcolor[HTML]{D9D9D9} 
OmniQ                & \multicolumn{1}{c|}{\cellcolor[HTML]{D9D9D9}2} & 15.5          & 13.2          & 8.71          & 7.58          & 37.4          & 17.2          & \multicolumn{1}{c|}{\cellcolor[HTML]{D9D9D9}7.81}          & 24.9          & 18.3          & 13.9          & 10.8          & 90.6          & 26.8          & 12.3          \\
\rowcolor[HTML]{D9D9D9} 
QuIP\# \tiny{no FT \& no $E_8$}                & \multicolumn{1}{c|}{\cellcolor[HTML]{D9D9D9}2} & 9.95          & 7.18          & 5.80          & 5.02          & 12.3          & 7.60          & \multicolumn{1}{c|}{\cellcolor[HTML]{D9D9D9}4.87}          & 11.7          & 8.67          & 7.55          & 6.83          & 14.8          & 9.57          & 6.82          \\
\rowcolor[HTML]{D9D9D9} 
QuIP\#               & \multicolumn{1}{c|}{\cellcolor[HTML]{D9D9D9}2} & \textbf{6.86} & \textbf{5.97} & \textbf{5.02} & \textbf{4.36} & \textbf{6.66} & \textbf{5.74} & \multicolumn{1}{c|}{\cellcolor[HTML]{D9D9D9}\textbf{4.16}} & \textbf{8.36} & \textbf{7.48} & \textbf{6.71} & \textbf{6.19} & \textbf{8.35} & \textbf{7.45} & \textbf{6.12} \\ \bottomrule
\end{tabular}
\end{table*}
\begin{table*}[ht]
\caption{Zeroshot Accuracy (\texttt{acc} in LM Eval, not \texttt{acc_norm}), Llama 2.}
\label{tab:zeroshot}
\centering
\small\sc
\tabcolsep=0.10cm
\centering
\begin{tabular}{cccccccccccccccc}
\multicolumn{1}{l}{}                                & \multicolumn{5}{c}{2-70}                                                                                         & \multicolumn{5}{c}{2-13}                                                                                         & \multicolumn{5}{c}{2-7}                                             \\ \hline
\multicolumn{1}{c|}{Method}                         & Bits & ArcC          & ArcE          & PiQA          & \multicolumn{1}{c|}{Wino}                                  & Bits & ArcC          & ArcE          & PiQA          & \multicolumn{1}{c|}{Wino}                                  & Bits & ArcC          & ArcE          & PiQA          & Wino          \\ \hline
\multicolumn{1}{c|}{FP16}                           & 16   & 51.1          & 77.7          & 81.1          & \multicolumn{1}{c|}{77.0}                                  & 16   & 45.6          & 73.3          & 73.5          & \multicolumn{1}{c|}{69.6}                                  & 16   & 40.0          & 69.3          & 78.5          & 67.3          \\
\rowcolor[HTML]{D9D9D9} 
\multicolumn{1}{c|}{\cellcolor[HTML]{D9D9D9}OmniQ}  & 4    & 49.8          & 77.9          & 80.7          & \multicolumn{1}{c|}{\cellcolor[HTML]{D9D9D9}75.8}          & 4    & 43.1          & 70.2          & 78.4          & \multicolumn{1}{c|}{\cellcolor[HTML]{D9D9D9}67.8}          & 4    & 37.9          & 67.8          & 77.1          & 67.0          \\
\rowcolor[HTML]{D9D9D9} 
\multicolumn{1}{c|}{\cellcolor[HTML]{D9D9D9}QuIP}   & 4    & 47.0          & 74.3          & 80.3          & \multicolumn{1}{c|}{\cellcolor[HTML]{D9D9D9}76.0}          & 4    & 44.9          & 73.3          & \textbf{79.0} & \multicolumn{1}{c|}{\cellcolor[HTML]{D9D9D9}69.7}          & 4    & -             & -             & -             & -             \\
\rowcolor[HTML]{D9D9D9} 
\multicolumn{1}{c|}{\cellcolor[HTML]{D9D9D9}AQLM}   & 4.07 & 51.0          & 78.1          & 81.4          & \multicolumn{1}{c|}{\cellcolor[HTML]{D9D9D9}76.9}          & 3.94 & 43.9          & 72.2          & 78.6          & \multicolumn{1}{c|}{\cellcolor[HTML]{D9D9D9}70.4}          & 4.04 & 40.3          & 68.9          & 77.7          & 67.3          \\
\rowcolor[HTML]{D9D9D9} 
\multicolumn{1}{c|}{\cellcolor[HTML]{D9D9D9}QuIP\#} & 4    & 50.6          & \textbf{78.1} & \textbf{81.4} & \multicolumn{1}{c|}{\cellcolor[HTML]{D9D9D9}\textbf{77.1}} & 4    & \textbf{45.5} & \textbf{73.9} & 78.9          & \multicolumn{1}{c|}{\cellcolor[HTML]{D9D9D9}69.9}          & 4    & \textbf{40.5} & \textbf{69.1} & \textbf{78.4} & \textbf{67.6} \\
\multicolumn{1}{c|}{OmniQ}                          & 3    & 47.6          & 75.7          & 79.7          & \multicolumn{1}{c|}{73.5}                                  & 3    & 42.0          & 69.0          & 77.7          & \multicolumn{1}{c|}{65.9}                                  & 3    & 35.3          & 62.6          & 73.6          & 63.6          \\
\multicolumn{1}{c|}{QuIP}                           & 3    & 46.3          & 73.2          & 80.0          & \multicolumn{1}{c|}{74.6}                                  & 3    & 41.5          & 70.4          & 76.9          & \multicolumn{1}{c|}{\textbf{69.9}}                         & 3    & -             & -             & -             & -             \\
\multicolumn{1}{c|}{AQLM}                           & 3.01 & 50.0          & 77.6          & 81.3          & \multicolumn{1}{c|}{\textbf{77.2}}                         & 3.03 & 43.6          & \textbf{73.5} & 77.8          & \multicolumn{1}{c|}{67.6}                                  & 3.04 & 38.7          & 67.8          & 76.6          & \textbf{68.4} \\
\multicolumn{1}{c|}{QuIP\#}                         & 3    & \textbf{50.9} & \textbf{77.7} & \textbf{81.4} & \multicolumn{1}{c|}{76.4}                                  & 3    & \textbf{44.0} & 72.5          & \textbf{78.4} & \multicolumn{1}{c|}{69.1}                                  & 3    & \textbf{39.2} & \textbf{68.4} & \textbf{77.3} & 66.5          \\
\rowcolor[HTML]{D9D9D9} 
\multicolumn{1}{c|}{\cellcolor[HTML]{D9D9D9}OmniQ}  & 2    & 28.7          & 55.4          & 68.8          & \multicolumn{1}{c|}{\cellcolor[HTML]{D9D9D9}53.2}          & 2    & 23.0          & 44.4          & 62.6          & \multicolumn{1}{c|}{\cellcolor[HTML]{D9D9D9}52.6}          & 2    & 21.6          & 35.2          & 57.5          & 51.5          \\
\rowcolor[HTML]{D9D9D9} 
\multicolumn{1}{c|}{\cellcolor[HTML]{D9D9D9}QuIP}   & 2    & 34.0          & 62.2          & 74.8          & \multicolumn{1}{c|}{\cellcolor[HTML]{D9D9D9}67.5}          & 2    & 23.5          & 45.2          & 62.0          & \multicolumn{1}{c|}{\cellcolor[HTML]{D9D9D9}52.8}          & 2    & 19.4          & 26.0          & 54.6          & 51.8          \\
\rowcolor[HTML]{D9D9D9} 
\multicolumn{1}{c|}{\cellcolor[HTML]{D9D9D9}AQLM}   & 2.07 & 47.9          & 77.7          & 80.4          & \multicolumn{1}{c|}{\cellcolor[HTML]{D9D9D9}75.9}          & 1.97 & 38.5          & 67.0          & 75.1          & \multicolumn{1}{c|}{\cellcolor[HTML]{D9D9D9}\textbf{69.5}} & 2.02 & 33.6          & 62.8          & 73.5          & 64.6          \\
\rowcolor[HTML]{D9D9D9} 
\multicolumn{1}{c|}{\cellcolor[HTML]{D9D9D9}QuIP\#} & 2    & \textbf{48.7} & 77.3          & 80.3          & \multicolumn{1}{c|}{\cellcolor[HTML]{D9D9D9}\textbf{75.9}} & 2    & \textbf{39.5} & \textbf{69.3} & \textbf{77.3} & \multicolumn{1}{c|}{\cellcolor[HTML]{D9D9D9}67.7}          & 2    & \textbf{34.6} & \textbf{64.6} & \textbf{75.1} & \textbf{64.9} \\ \hline
\end{tabular}
\end{table*}
\begin{table}[ht]
\centering
\caption{Wikitext2 and C4 perplexity ($\downarrow$), context length 4096.}
\label{tab:ppl4096}
\small\sc
\tabcolsep=0.07cm
\renewcommand{\arraystretch}{1.05}
\begin{tabular}{@{}lccccccccc@{}}
                                                    & \multicolumn{1}{l}{} & 2-7           & \multicolumn{1}{l}{}                                       & \multicolumn{1}{l}{} & 2-13          & \multicolumn{1}{l}{}                                       & \multicolumn{1}{l}{} & 2-70          & \multicolumn{1}{l}{} \\ \midrule
\multicolumn{1}{c|}{Method}                         & Bits                 & W2            & \multicolumn{1}{c|}{C4}                                    & Bits                 & W2            & \multicolumn{1}{c|}{C4}                                    & Bits                 & W2            & C4                   \\ \midrule
\multicolumn{1}{l|}{FP16}                           & 16                   & 5.12          & \multicolumn{1}{c|}{6.63}                                  & 16                   & 4.57          & \multicolumn{1}{c|}{6.05}                                  & 16                   & 3.12          & 4.97                 \\
\rowcolor[HTML]{D9D9D9} 
\multicolumn{1}{l|}{\cellcolor[HTML]{D9D9D9}QuIP\#} & 4                    & \textbf{5.19} & \multicolumn{1}{c|}{\cellcolor[HTML]{D9D9D9}6.75}          & 4                    & 4.63          & \multicolumn{1}{c|}{\cellcolor[HTML]{D9D9D9}6.13}          & 4                    & 3.18          & 5.02                 \\
\rowcolor[HTML]{D9D9D9} 
\multicolumn{1}{l|}{\cellcolor[HTML]{D9D9D9}$\drsh$ no FT}  & 4                    & 5.22          & \multicolumn{1}{c|}{\cellcolor[HTML]{D9D9D9}6.79}          & 4                    & 4.65          & \multicolumn{1}{c|}{\cellcolor[HTML]{D9D9D9}6.15}          & 4                    & 3.18          & 5.02                 \\
\rowcolor[HTML]{D9D9D9} 
\multicolumn{1}{l|}{\cellcolor[HTML]{D9D9D9}\hspace{2pt} $\drsh$ no $E_8$} & 4                    & 5.29          & \multicolumn{1}{c|}{\cellcolor[HTML]{D9D9D9}6.86}          & 4                    & 4.68          & \multicolumn{1}{c|}{\cellcolor[HTML]{D9D9D9}6.20}          & 4                    & 3.22          & 5.05                 \\
\rowcolor[HTML]{D9D9D9} 
\multicolumn{1}{l|}{\cellcolor[HTML]{D9D9D9}QuIP}   & 4                    & -             & \multicolumn{1}{c|}{\cellcolor[HTML]{D9D9D9}-}             & 4                    & 4.76          & \multicolumn{1}{c|}{\cellcolor[HTML]{D9D9D9}6.29}          & 4                    & 3.58          & 5.38                 \\
\rowcolor[HTML]{D9D9D9} 
\multicolumn{1}{l|}{\cellcolor[HTML]{D9D9D9}AQLM}   & 4.04                 & 5.21          & \multicolumn{1}{c|}{\cellcolor[HTML]{D9D9D9}6.74}          & 3.94                 & 4.64          & \multicolumn{1}{c|}{\cellcolor[HTML]{D9D9D9}6.14}          & 4.07                 & 3.17          & 5.01                 \\
\rowcolor[HTML]{FFFFFF} 
\multicolumn{1}{l|}{\cellcolor[HTML]{FFFFFF}QuIP\#} & 3                    & \textbf{5.41} & \multicolumn{1}{c|}{\cellcolor[HTML]{FFFFFF}\textbf{7.04}} & 3                    & \textbf{4.78} & \multicolumn{1}{c|}{\cellcolor[HTML]{FFFFFF}\textbf{6.35}} & 3                    & \textbf{3.35} & \textbf{5.15}        \\
\rowcolor[HTML]{FFFFFF} 
\multicolumn{1}{l|}{\cellcolor[HTML]{FFFFFF}$\drsh$ no FT}  & 3                    & 5.60          & \multicolumn{1}{c|}{\cellcolor[HTML]{FFFFFF}7.34}          & 3                    & 4.90          & \multicolumn{1}{c|}{\cellcolor[HTML]{FFFFFF}6.50}          & 3                    & 3.41          & 5.20                 \\
\rowcolor[HTML]{FFFFFF} 
\multicolumn{1}{l|}{\cellcolor[HTML]{FFFFFF}\hspace{2pt} $\drsh$ no $E_8$} & 3                    & 5.77          & \multicolumn{1}{c|}{\cellcolor[HTML]{FFFFFF}7.61}          & 3                    & 4.99          & \multicolumn{1}{c|}{\cellcolor[HTML]{FFFFFF}6.65}          & 3                    & 3.48          & 5.28                 \\
\rowcolor[HTML]{FFFFFF} 
\multicolumn{1}{l|}{\cellcolor[HTML]{FFFFFF}QuIP}   & 3                    & -             & \multicolumn{1}{c|}{\cellcolor[HTML]{FFFFFF}-}             & 3                    & 5.12          & \multicolumn{1}{c|}{\cellcolor[HTML]{FFFFFF}6.79}          & 3                    & 3.87          & 5.67                 \\
\rowcolor[HTML]{FFFFFF} 
\multicolumn{1}{l|}{\cellcolor[HTML]{FFFFFF}AQLM}   & 3.04                 & 5.46          & \multicolumn{1}{c|}{\cellcolor[HTML]{FFFFFF}7.10}          & 3.03                 & 4.83          & \multicolumn{1}{c|}{\cellcolor[HTML]{FFFFFF}6.37}          & 3.01                 & 3.36          & 5.17                 \\
\rowcolor[HTML]{D9D9D9} 
\multicolumn{1}{l|}{\cellcolor[HTML]{D9D9D9}QuIP\#} & 2                    & \textbf{6.19} & \multicolumn{1}{c|}{\cellcolor[HTML]{D9D9D9}\textbf{8.16}} & 2                    & \textbf{5.35} & \multicolumn{1}{c|}{\cellcolor[HTML]{D9D9D9}\textbf{7.20}} & 2                    & \textbf{3.91} & \textbf{5.71}        \\
\rowcolor[HTML]{D9D9D9} 
\multicolumn{1}{l|}{\cellcolor[HTML]{D9D9D9}$\drsh$ no FT}  & 2                    & 8.22          & \multicolumn{1}{c|}{\cellcolor[HTML]{D9D9D9}11.0}          & 2                    & 6.06          & \multicolumn{1}{c|}{\cellcolor[HTML]{D9D9D9}8.07}          & 2                    & 4.16          & 6.01                 \\
\rowcolor[HTML]{D9D9D9} 
\multicolumn{1}{l|}{\cellcolor[HTML]{D9D9D9}\hspace{2pt} $\drsh$ no $E_8$} & 2                    & 11.2          & \multicolumn{1}{c|}{\cellcolor[HTML]{D9D9D9}14.5}          & 2                    & 7.04          & \multicolumn{1}{c|}{\cellcolor[HTML]{D9D9D9}9.37}          & 2                    & 4.58          & 6.51                 \\
\rowcolor[HTML]{D9D9D9} 
\multicolumn{1}{l|}{\cellcolor[HTML]{D9D9D9}QuIP}   & 2                    & -             & \multicolumn{1}{c|}{\cellcolor[HTML]{D9D9D9}-}             & 2                    & 13.5          & \multicolumn{1}{c|}{\cellcolor[HTML]{D9D9D9}16.2}          & 2                    & 5.90          & 8.17                 \\
\rowcolor[HTML]{D9D9D9} 
\multicolumn{1}{l|}{\cellcolor[HTML]{D9D9D9}AQLM}   & 2.02                 & 6.93          & \multicolumn{1}{c|}{\cellcolor[HTML]{D9D9D9}8.84}          & 1.97                 & 5.70          & \multicolumn{1}{c|}{\cellcolor[HTML]{D9D9D9}7.59}          & 2.07                 & 3.94          & 5.72                 \\ \bottomrule
\end{tabular}

\end{table}
\begin{table}[]
\centering
\caption{\qss generation throughput on a NVIDIA RTX 4090 using the \textit{FlashAttention} library's Llama implementation. \qss achieves $>50\%$ peak memory bandwidth (1TB/s) during generation and admits fast inference.}
\label{tab:genspeed}
\small\sc
\begin{tabular}{@{}ccccc@{}}
\toprule
Model & \begin{tabular}[c]{@{}c@{}}2 Bit\\ tok/s\end{tabular} & \begin{tabular}[c]{@{}c@{}}2 Bit \%\\ Mem BW\end{tabular} & \begin{tabular}[c]{@{}c@{}}4 Bit\\ tok/s\end{tabular} & \begin{tabular}[c]{@{}c@{}}4 Bit \%\\ Mem BW\end{tabular} \\ \midrule
2-7B                                                         & 170.50                                                & 29.60\%                                                   & 117.73                                                & 40.87\%                                                   \\
2-13B                                                        & 104.83                                                & 33.80\%                                                   & 71.09                                                 & 45.84\%                                                   \\
1-30B                                                        & 51.60                                                 & 38.39\%                                                   & 32.50                                                 & 48.36\%                                                   \\
2-70B                                                        & 32.74                                                 & 56.84\%                                                   & OOM                                                   & OOM                                                       \\ \bottomrule
\end{tabular}
\end{table}

\begin{table}[]
\centering
\caption{\qss vs AQLM and FP16 generation throughput on a NVIDIA RTX 4090 using the \textit{HuggingFace} library's Llama implementation. Unlike AQLM, whose codebook is too large to fit in L1 cache, \qss achieves significant speedups over FP16.}
\label{tab:genspeedaqlm}
\small\sc
\begin{tabular}{@{}ccc@{}}
\toprule
Method       & 2-7B       & 2-70B \\ \midrule
FP16         & 33.1 tok/s & OOM    \\
AQLM 2 Bit   & 20.6       & 8.27  \\
QuIP\# 2 Bit & 106.3      & 25.9  \\ \bottomrule
\end{tabular}
\end{table}

\begin{figure}
\includegraphics[width=0.9\linewidth]{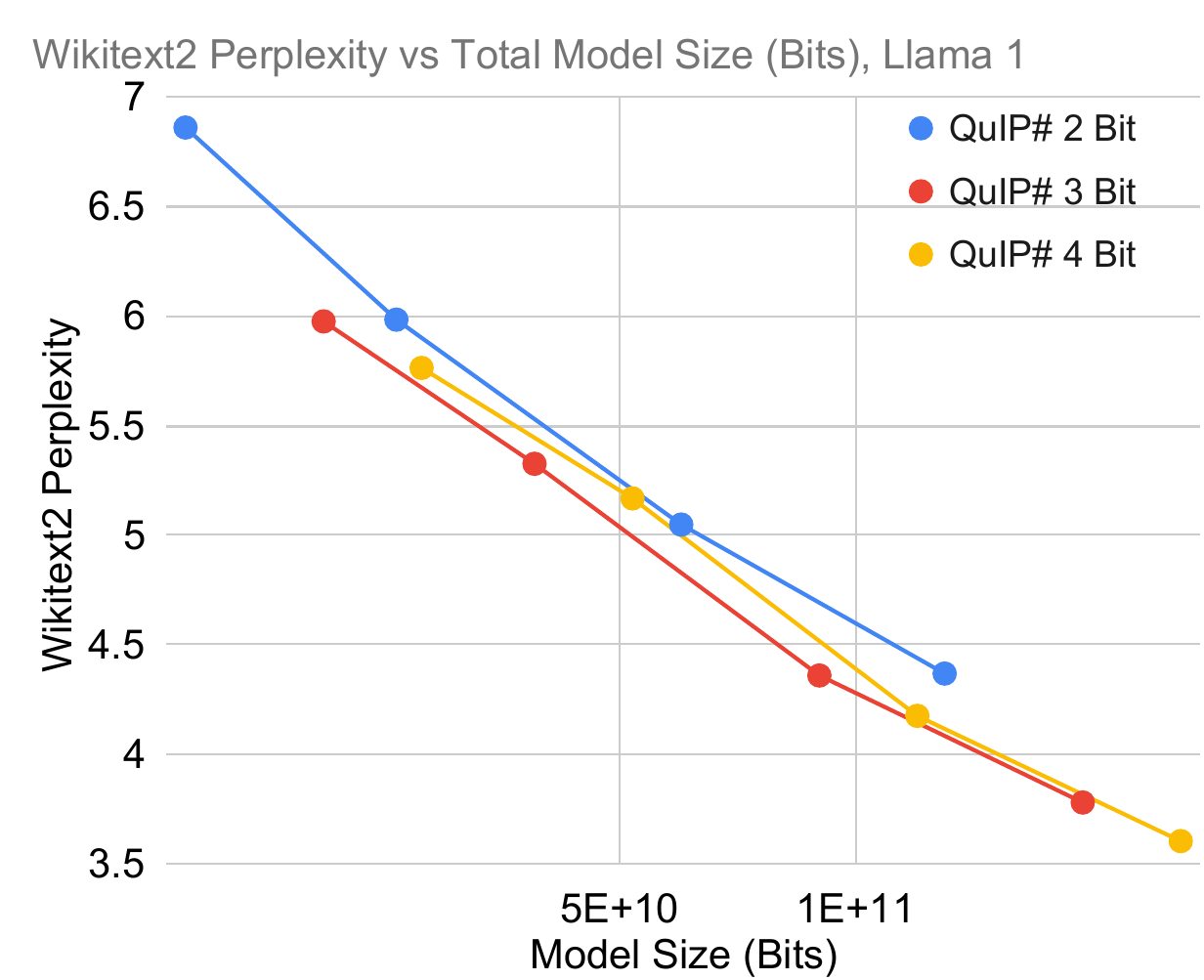}
\vspace{-0.25cm}
\caption{\qss scaling, Llama 1. Like Llama 2, \qss 3 bit scales better than \qss 4 bit for Llama 1 models and \qss 2 bit scales similarly to higher bitrates.}
\vspace{-0.5cm}
\label{fig:llama1scale}
\end{figure}

Our main experiments show the performance of \qss on the Llama 1 \cite{llama1} and 2 \cite{llama2} family of models.
These models range in size from 7 billion to 70 billion parameters and offer good performance, making them suitable for understanding how quantization methods perform and scale.
Additional results for other models are available in the Appendix.

In Section \ref{sec:qss_llama}, we compare \qss with recently published weight-only PTQ methods.
AWQ scales weights by activation magnitudes \textit{before} quantizing to reduce outliers \cite{awq}.
OmniQuant learns model-preserving layerwise transformations that reduce outliers per transformer block \cite{omniquant}.
AQLM uses vector quantization with learnable unstructured 8D codebooks \cite{aqlm}\footnote{We report results from the Jan 11, 2024 ArXiv version.}.
We report AQLM's ``$1 \times 16$'' numbers, which amounts to using a single codebook with $2^{16}$ entries $\in \mathbb{R}^8$ \textit{per linear layer}. 
These codebooks each take up 1MiB of space, making them too large to fit in the L1 cache of any current GPU and thus preventing fast inference (see Table \ref{tab:genspeedaqlm}).
Finally, we include QuIP \cite{chee2023quip} as a baseline for the improvements in \qs.

We report W$x$A16 numbers for AWQ and OmniQuant from the OmniQuant paper and AQLM numbers from AQLM.
We note that there are currently 2 methods for evaluating perplexity: using the Llama 1 context length of 2048 or using the model's native context length (e.g. 4096 for Llama 2).
OmniQuant and AWQ use 2048 for Llama 2 while AQLM uses 4096; we report both sets of numbers.
We also note that AQLM paper reports \qss numbers from an outdated version of \qs; the numbers here represent the latest \qss numbers.
Finally, we \textbf{bold} numbers in our tables when they are clearly better, such as a smaller model matching or outperforming a larger model or a similar sized model significantly outperforming another model.

\subsection{\qss on Llama Models}
\label{sec:qss_llama}


Table \ref{tab:ppl2048} shows a comparison of \qss with OmniQuant, AWQ, and \qss without fine tuning and E8P, with context length 2048.
\qss offers a paradigm shift in quantization quality over OmniQuant and AWQ.
Notably, while AWQ falls apart at even 2.15 bits \cite{omniquant} and OmniQuant produces unusable models at 2 bits, \qss produces high quality models that are close to OmniQuant 3 bit models.
Table \ref{tab:ppl2048} also shows the importance of incoherence processing.
\qss without fine-tuning or lattice codebooks significantly outperforms OmniQuant and AWQ, which both rely on heuristics to reduce model outliers during quantization. 

Table \ref{tab:ppl4096} shows a comparison of \qss with AQLM with context length 4096.
At 2 and 3 bits, \qss either significantly outperforms similar-sized AQLM models or achieves similar performance with a smaller model\footnote{In our experience, at extreme quantization levels, even 0.1 bits can make a significant difference in quantization quality.}.
At 4 bits, both methods perform similarly.
This is not surprising as \sota 4 bit models are all very close to FP16 performance.
Furthermore, the \qss 3 and 4 bit results presented in this paper use residual vector quantization; one could potentially achieve better numbers with more advanced multi-codebook quantization approaches.

Table \ref{tab:zeroshot} shows zeroshot results for \qs, AQLM, and OmniQuant.
Both AQLM and \qss signficantly outperform OmniQuant, which correlates with the perpelxity results.
AQLM and \qss both perform very close to FP16 at higher bitrates and for larger models, but \qss tends to outperform AQLM at lower bitrates and model sizes.
We note that zeroshot tasks have an element of randomness and even FP16 numbers can disagree by up to $0.5\%$.

\subsection{\qss Bit Scaling}


Figures \ref{fig:llama2scale} (first page) and \ref{fig:llama1scale} show how \qss scales on the Llama family of models and Wikitext2.
On both Llama 1 and 2, \qss 3 bit outperforms \qss 4 bit and \qss 2 bit offers similar scaling to 3 and 4 bit models.
Furthermore, on Llama 2, \qss 3 bit outperforms a theoretical lossless 4 bit model (FP16 at 4 bits).
To the best of our knowledge, this is the first time a 3 bit PTQ method has outperformed a theoretical lossless 4 bit model and also the first time a 2 bit PTQ method has offered similar scaling to higher bitrates.

\subsection{Efficient Inference with \qs}

One of the key benefits of PTQ is to increase the maximum possible inference throughput on a given device.
Since small-batch autoregressive decoding is usually memory bound, a smaller model requires less data to be read and can therefore be served faster.
However, achieving an actual speedup requires a quantization method with low decoding overhead, or inference will be bottlenecked by decoding.
For example, the AQLM models in the experiment tables use a different $2^{16} \times 8$ codebook for every linear layer. 
Each entry in these codebooks takes 2 bytes, meaning that each codebook is 1MiB. 
During inference, weights are read from these codebook in an essentialy random access pattern, meaning that the entire codebook must fit in L1 cache to enable fast inference (even L2 cache is too slow).
However, 1MiB is larger than any current GPU's L1 cache (the H100 has 256KB), so AQLM inference suffers from high cache miss rates and is actually \textit{slower than FP16} on modern GPUs (Table \ref{tab:genspeedaqlm}).

In contrast, \qss was designed around fast inference. 
The RHT can be computed in essentially $O(n \log n)$ time and E8P only requires 1KiB and can be decoded from with very few ($<5$) instructions per weight.
Table \ref{tab:genspeed} shows \qs's generation speed as measured with the FlashAttention library's \cite{dao2022flashattention, dao2023flashattention2} implementation of Llama.
\qss is able to achieve over 50\% of peak memory bandwidth with a 2 bit model even with minimal kernel fusion in the RHT, validating our design choices.
We note that since these ``fast inference design choices'' essentially amount to restrictions on what can be done during quantization, it should be entirely possible to achieve even better quantization quality at the expense of inference speed.

Finally, if we look at the speed-quality tradeoffs of different quantization methods, we also find that \qss enables new frontiers of PTQ performance. 
Compared to QuIP's published throughput numbers (measured on an A6000), \qss on an A6000 achieves roughly twice the inference throughput at the same bitrate, making \qss strictly better. 
Compared to existing ``fast inference'' quantization methods such as SpQR \cite{dettmers2023spqr} and SqueezeLLM \cite{kim2023squeezellm}, we again find that \qss offers significantly higher throughput ($>40\%$) at the same or better quantization quality.

\subsection{Ablations}

\label{sec:ablations}

Table \ref{tab:ppl4096} also contains an ablation on the various components of \qs. 
The ``no FT'' row shows \qss without fine-tuning and the ``no $E_8$'' row shows \qss without fine-tuning and lattice codebooks.
For the latter, we round to the 1-dimensional half-integer grid.
We also include QuIP numbers as reported by AQLM.
At all bitrates, each component of \qss brings additional performance gains.
The difference between QuIP and \qss without fine-tuning and lattice codebooks also shows the difference between QuIP's Kronecker factorization and \qs's RHT.
The RHT offers stronger incoherence properties than the Kronecker factorization (Section \ref{sec:RHT}), which improves performance.

\section{Conclusion}
We present \qs, a weight-only post training compression method that achieves state-of-the-art results on LLMs at 2, 3, and 4 bits per weight.
\qss uses the Randomized Hadamard Transform as an efficient and principled form of outlier suppression, and introduces the $E_8$ lattice-based E8P codebook to better quantize RHT transformed weights.
The E8P codebook is highly symmetric and admits fast inference, allowing a ``proof of concept'' \qss CUDA implementation to achieve over 50\% peak memory bandwidth on modern GPUs.
\qss also implements inter-layer fine tuning, further improving quantization.
To the best of our knowledge, \qss is the first PTQ method to achieve superior scaling at 3 bits over 4 bits and similar scaling at 2 bits to higher bitrates.
Our results indicate that, in the near future, 2 bit models are likely to scale better than 3 bit ones.

\section*{Impact Statement}
This paper presents work whose goal is to advance the field of Machine Learning. There are many potential societal consequences of our work, none which we feel must be specifically highlighted here.

\section*{Acknowledgements}

We thank, in no particular order, David Hou for helping with the \qss CUDA implementation, Tiancheng Yuan for lending his RTX 4090 and helping with acquiring \qss timing numbers, Tri Dao for a fast CUDA implementation of the Hadamard transform and general help with \qs, and Together AI for compute resources.



\bibliography{main}

\begin{thebibliography}{34}
\providecommand{\natexlab}[1]{#1}
\providecommand{\url}[1]{\texttt{#1}}
\expandafter\ifx\csname urlstyle\endcsname\relax
  \providecommand{\doi}[1]{doi: #1}\else
  \providecommand{\doi}{doi: \begingroup \urlstyle{rm}\Url}\fi

\bibitem[Almazrouei et~al.(2023)Almazrouei, Alobeidli, Alshamsi, Cappelli, Cojocaru, Debbah, Étienne Goffinet, Hesslow, Launay, Malartic, Mazzotta, Noune, Pannier, and Penedo]{falcon}
Almazrouei, E., Alobeidli, H., Alshamsi, A., Cappelli, A., Cojocaru, R., Debbah, M., Étienne Goffinet, Hesslow, D., Launay, J., Malartic, Q., Mazzotta, D., Noune, B., Pannier, B., and Penedo, G.
\newblock The falcon series of open language models, 2023.

\bibitem[Chee et~al.(2022)Chee, Renz, Damle, and Sa]{jerrypruning}
Chee, J., Renz, M., Damle, A., and Sa, C.~D.
\newblock Model preserving compression for neural networks.
\newblock In Oh, A.~H., Agarwal, A., Belgrave, D., and Cho, K. (eds.), \emph{Advances in Neural Information Processing Systems}, 2022.
\newblock URL \url{https://openreview.net/forum?id=gt-l9Hu2ndd}.

\bibitem[Chee et~al.(2023)Chee, Cai, Kuleshov, and Sa]{chee2023quip}
Chee, J., Cai, Y., Kuleshov, V., and Sa, C.~D.
\newblock Qu{IP}: 2-bit quantization of large language models with guarantees.
\newblock In \emph{Thirty-seventh Conference on Neural Information Processing Systems}, 2023.
\newblock URL \url{https://openreview.net/forum?id=xrk9g5vcXR}.

\bibitem[Cochran et~al.(1967)Cochran, Cooley, Favin, Helms, Kaenel, Lang, Maling, Nelson, Rader, and Welch]{fft}
Cochran, W., Cooley, J., Favin, D., Helms, H., Kaenel, R., Lang, W., Maling, G., Nelson, D., Rader, C., and Welch, P.
\newblock What is the fast fourier transform?
\newblock \emph{Proceedings of the IEEE}, 55\penalty0 (10):\penalty0 1664--1674, 1967.
\newblock \doi{10.1109/PROC.1967.5957}.

\bibitem[Computer(2023)]{rpv1}
Computer, T.
\newblock Redpajama: An open source recipe to reproduce llama training dataset, 2023.
\newblock URL \url{https://github.com/togethercomputer/RedPajama-Data}.

\bibitem[Dao(2023)]{dao2023flashattention2}
Dao, T.
\newblock Flash{A}ttention-2: Faster attention with better parallelism and work partitioning.
\newblock 2023.

\bibitem[Dao et~al.(2022)Dao, Fu, Ermon, Rudra, and R{\'e}]{dao2022flashattention}
Dao, T., Fu, D.~Y., Ermon, S., Rudra, A., and R{\'e}, C.
\newblock Flash{A}ttention: Fast and memory-efficient exact attention with {IO}-awareness.
\newblock In \emph{Advances in Neural Information Processing Systems}, 2022.

\bibitem[Dettmers \& Zettlemoyer(2023)Dettmers and Zettlemoyer]{dettmers2022scaling}
Dettmers, T. and Zettlemoyer, L.
\newblock The case for 4-bit precision: k-bit inference scaling laws.
\newblock In Krause, A., Brunskill, E., Cho, K., Engelhardt, B., Sabato, S., and Scarlett, J. (eds.), \emph{Proceedings of the 40th International Conference on Machine Learning}, volume 202 of \emph{Proceedings of Machine Learning Research}, pp.\  7750--7774. PMLR, 23--29 Jul 2023.
\newblock URL \url{https://proceedings.mlr.press/v202/dettmers23a.html}.

\bibitem[Dettmers et~al.(2023)Dettmers, Svirschevski, Egiazarian, Kuznedelev, Frantar, Ashkboos, Borzunov, Hoefler, and Alistarh]{dettmers2023spqr}
Dettmers, T., Svirschevski, R., Egiazarian, V., Kuznedelev, D., Frantar, E., Ashkboos, S., Borzunov, A., Hoefler, T., and Alistarh, D.
\newblock Spqr: A sparse-quantized representation for near-lossless llm weight compression, 2023.

\bibitem[Egiazarian et~al.(2024)Egiazarian, Panferov, Kuznedelev, Frantar, Babenko, and Alistarh]{aqlm}
Egiazarian, V., Panferov, A., Kuznedelev, D., Frantar, E., Babenko, A., and Alistarh, D.
\newblock Extreme compression of large language models via additive quantization, 2024.

\bibitem[Fino \& Algazi(1976)Fino and Algazi]{fwht}
Fino and Algazi.
\newblock Unified matrix treatment of the fast walsh-hadamard transform.
\newblock \emph{IEEE Transactions on Computers}, C-25\penalty0 (11):\penalty0 1142--1146, 1976.
\newblock \doi{10.1109/TC.1976.1674569}.

\bibitem[Frantar et~al.(2023)Frantar, Ashkboos, Hoefler, and Alistarh]{optq}
Frantar, E., Ashkboos, S., Hoefler, T., and Alistarh, D.
\newblock {OPTQ}: Accurate quantization for generative pre-trained transformers.
\newblock In \emph{The Eleventh International Conference on Learning Representations}, 2023.
\newblock URL \url{https://openreview.net/forum?id=tcbBPnfwxS}.

\bibitem[Gao et~al.(2023)Gao, Tow, Abbasi, Biderman, Black, DiPofi, Foster, Golding, Hsu, Le~Noac'h, Li, McDonell, Muennighoff, Ociepa, Phang, Reynolds, Schoelkopf, Skowron, Sutawika, Tang, Thite, Wang, Wang, and Zou]{lmeval}
Gao, L., Tow, J., Abbasi, B., Biderman, S., Black, S., DiPofi, A., Foster, C., Golding, L., Hsu, J., Le~Noac'h, A., Li, H., McDonell, K., Muennighoff, N., Ociepa, C., Phang, J., Reynolds, L., Schoelkopf, H., Skowron, A., Sutawika, L., Tang, E., Thite, A., Wang, B., Wang, K., and Zou, A.
\newblock A framework for few-shot language model evaluation, 12 2023.
\newblock URL \url{https://zenodo.org/records/10256836}.

\bibitem[Halko et~al.(2011)Halko, Martinsson, and Tropp]{halko2011finding}
Halko, N., Martinsson, P.-G., and Tropp, J.~A.
\newblock Finding structure with randomness: Probabilistic algorithms for constructing approximate matrix decompositions.
\newblock \emph{SIAM review}, 53\penalty0 (2):\penalty0 217--288, 2011.

\bibitem[Hedayat \& Wallis(1978)Hedayat and Wallis]{hadconj}
Hedayat, A. and Wallis, W.~D.
\newblock {Hadamard Matrices and Their Applications}.
\newblock \emph{The Annals of Statistics}, 6\penalty0 (6):\penalty0 1184 -- 1238, 1978.
\newblock \doi{10.1214/aos/1176344370}.
\newblock URL \url{https://doi.org/10.1214/aos/1176344370}.

\bibitem[Jiang et~al.(2024)Jiang, Sablayrolles, Roux, Mensch, Savary, Bamford, Chaplot, de~las Casas, Hanna, Bressand, Lengyel, Bour, Lample, Lavaud, Saulnier, Lachaux, Stock, Subramanian, Yang, Antoniak, Scao, Gervet, Lavril, Wang, Lacroix, and Sayed]{mixtral}
Jiang, A.~Q., Sablayrolles, A., Roux, A., Mensch, A., Savary, B., Bamford, C., Chaplot, D.~S., de~las Casas, D., Hanna, E.~B., Bressand, F., Lengyel, G., Bour, G., Lample, G., Lavaud, L.~R., Saulnier, L., Lachaux, M.-A., Stock, P., Subramanian, S., Yang, S., Antoniak, S., Scao, T.~L., Gervet, T., Lavril, T., Wang, T., Lacroix, T., and Sayed, W.~E.
\newblock Mixtral of experts, 2024.

\bibitem[Juang \& Gray(1982)Juang and Gray]{rvq}
Juang, B.-H. and Gray, A.
\newblock Multiple stage vector quantization for speech coding.
\newblock In \emph{ICASSP '82. IEEE International Conference on Acoustics, Speech, and Signal Processing}, volume~7, pp.\  597--600, 1982.
\newblock \doi{10.1109/ICASSP.1982.1171604}.

\bibitem[Kim et~al.(2023)Kim, Hooper, Gholami, Dong, Li, Shen, Mahoney, and Keutzer]{kim2023squeezellm}
Kim, S., Hooper, C., Gholami, A., Dong, Z., Li, X., Shen, S., Mahoney, M., and Keutzer, K.
\newblock Squeezellm: Dense-and-sparse quantization.
\newblock \emph{arXiv}, 2023.

\bibitem[Kingma \& Ba(2017)Kingma and Ba]{adam}
Kingma, D.~P. and Ba, J.
\newblock Adam: A method for stochastic optimization, 2017.

\bibitem[Kostina \& Verdú(2011)Kostina and Verdú]{kostina}
Kostina, V. and Verdú, S.
\newblock Fixed-length lossy compression in the finite blocklength regime: Gaussian source.
\newblock \emph{2011 IEEE Information Theory Workshop, ITW 2011}, 10 2011.
\newblock \doi{10.1109/ITW.2011.6089501}.

\bibitem[Lin et~al.(2023)Lin, Tang, Tang, Yang, Dang, Gan, and Han]{awq}
Lin, J., Tang, J., Tang, H., Yang, S., Dang, X., Gan, C., and Han, S.
\newblock Awq: Activation-aware weight quantization for llm compression and acceleration, 2023.

\bibitem[Liu et~al.(2023)Liu, Oguz, Zhao, Chang, Stock, Mehdad, Shi, Krishnamoorthi, and Chandra]{liu2023llmqat}
Liu, Z., Oguz, B., Zhao, C., Chang, E., Stock, P., Mehdad, Y., Shi, Y., Krishnamoorthi, R., and Chandra, V.
\newblock Llm-qat: Data-free quantization aware training for large language models, 2023.

\bibitem[Lloyd(1982)]{lloyd}
Lloyd, S.
\newblock Least squares quantization in pcm.
\newblock \emph{IEEE Transactions on Information Theory}, 28\penalty0 (2):\penalty0 129--137, 1982.
\newblock \doi{10.1109/TIT.1982.1056489}.

\bibitem[Nagel et~al.(2020)Nagel, Amjad, Van~Baalen, Louizos, and Blankevoort]{nagelround}
Nagel, M., Amjad, R.~A., Van~Baalen, M., Louizos, C., and Blankevoort, T.
\newblock Up or down? {A}daptive rounding for post-training quantization.
\newblock In III, H.~D. and Singh, A. (eds.), \emph{Proceedings of the 37th International Conference on Machine Learning}, volume 119 of \emph{Proceedings of Machine Learning Research}, pp.\  7197--7206. PMLR, 13--18 Jul 2020.
\newblock URL \url{https://proceedings.mlr.press/v119/nagel20a.html}.

\bibitem[Nagel et~al.(2022)Nagel, Fournarakis, Bondarenko, and Blankevoort]{nagelqat}
Nagel, M., Fournarakis, M., Bondarenko, Y., and Blankevoort, T.
\newblock Overcoming oscillations in quantization-aware training.
\newblock In Chaudhuri, K., Jegelka, S., Song, L., Szepesvari, C., Niu, G., and Sabato, S. (eds.), \emph{Proceedings of the 39th International Conference on Machine Learning}, volume 162 of \emph{Proceedings of Machine Learning Research}, pp.\  16318--16330. PMLR, 17--23 Jul 2022.
\newblock URL \url{https://proceedings.mlr.press/v162/nagel22a.html}.

\bibitem[Nguyen et~al.(2023)Nguyen, Poli, Faizi, Thomas, Birch-Sykes, Wornow, Patel, Rabideau, Massaroli, Bengio, Ermon, Baccus, and Ré]{hyenadna}
Nguyen, E., Poli, M., Faizi, M., Thomas, A., Birch-Sykes, C., Wornow, M., Patel, A., Rabideau, C., Massaroli, S., Bengio, Y., Ermon, S., Baccus, S.~A., and Ré, C.
\newblock Hyenadna: Long-range genomic sequence modeling at single nucleotide resolution.
\newblock 2023.

\bibitem[Rozière et~al.(2024)Rozière, Gehring, Gloeckle, Sootla, Gat, Tan, Adi, Liu, Sauvestre, Remez, Rapin, Kozhevnikov, Evtimov, Bitton, Bhatt, Ferrer, Grattafiori, Xiong, Défossez, Copet, Azhar, Touvron, Martin, Usunier, Scialom, and Synnaeve]{codellama}
Rozière, B., Gehring, J., Gloeckle, F., Sootla, S., Gat, I., Tan, X.~E., Adi, Y., Liu, J., Sauvestre, R., Remez, T., Rapin, J., Kozhevnikov, A., Evtimov, I., Bitton, J., Bhatt, M., Ferrer, C.~C., Grattafiori, A., Xiong, W., Défossez, A., Copet, J., Azhar, F., Touvron, H., Martin, L., Usunier, N., Scialom, T., and Synnaeve, G.
\newblock Code llama: Open foundation models for code, 2024.

\bibitem[Shao et~al.(2024)Shao, Chen, Zhang, Xu, Zhao, Li, Zhang, Gao, Qiao, and Luo]{omniquant}
Shao, W., Chen, M., Zhang, Z., Xu, P., Zhao, L., Li, Z., Zhang, K., Gao, P., Qiao, Y., and Luo, P.
\newblock Omniquant: Omnidirectionally calibrated quantization for large language models.
\newblock In \emph{The Twelfth International Conference on Learning Representations}, 2024.
\newblock URL \url{https://openreview.net/forum?id=8Wuvhh0LYW}.

\bibitem[Sloane()]{neilsloane}
Sloane, N.
\newblock {H}adamard {M}atrices --- neilsloane.com.
\newblock \url{http://neilsloane.com/hadamard/}.
\newblock [Accessed 02-02-2024].

\bibitem[Sun et~al.(2023)Sun, Liu, Bair, and Kolter]{wanda}
Sun, M., Liu, Z., Bair, A., and Kolter, J.~Z.
\newblock A simple and effective pruning approach for large language models.
\newblock In \emph{Workshop on Efficient Systems for Foundation Models @ ICML2023}, 2023.
\newblock URL \url{https://openreview.net/forum?id=tz9JV2PRSv}.

\bibitem[Touvron et~al.(2023{\natexlab{a}})Touvron, Lavril, Izacard, Martinet, Lachaux, Lacroix, Rozière, Goyal, Hambro, Azhar, Rodriguez, Joulin, Grave, and Lample]{llama1}
Touvron, H., Lavril, T., Izacard, G., Martinet, X., Lachaux, M.-A., Lacroix, T., Rozière, B., Goyal, N., Hambro, E., Azhar, F., Rodriguez, A., Joulin, A., Grave, E., and Lample, G.
\newblock Llama: Open and efficient foundation language models, 2023{\natexlab{a}}.

\bibitem[Touvron et~al.(2023{\natexlab{b}})Touvron, Martin, Stone, Albert, Almahairi, Babaei, Bashlykov, Batra, Bhargava, Bhosale, Bikel, Blecher, Ferrer, Chen, Cucurull, Esiobu, Fernandes, Fu, Fu, Fuller, Gao, Goswami, Goyal, Hartshorn, Hosseini, Hou, Inan, Kardas, Kerkez, Khabsa, Kloumann, Korenev, Koura, Lachaux, Lavril, Lee, Liskovich, Lu, Mao, Martinet, Mihaylov, Mishra, Molybog, Nie, Poulton, Reizenstein, Rungta, Saladi, Schelten, Silva, Smith, Subramanian, Tan, Tang, Taylor, Williams, Kuan, Xu, Yan, Zarov, Zhang, Fan, Kambadur, Narang, Rodriguez, Stojnic, Edunov, and Scialom]{llama2}
Touvron, H., Martin, L., Stone, K., Albert, P., Almahairi, A., Babaei, Y., Bashlykov, N., Batra, S., Bhargava, P., Bhosale, S., Bikel, D., Blecher, L., Ferrer, C.~C., Chen, M., Cucurull, G., Esiobu, D., Fernandes, J., Fu, J., Fu, W., Fuller, B., Gao, C., Goswami, V., Goyal, N., Hartshorn, A., Hosseini, S., Hou, R., Inan, H., Kardas, M., Kerkez, V., Khabsa, M., Kloumann, I., Korenev, A., Koura, P.~S., Lachaux, M.-A., Lavril, T., Lee, J., Liskovich, D., Lu, Y., Mao, Y., Martinet, X., Mihaylov, T., Mishra, P., Molybog, I., Nie, Y., Poulton, A., Reizenstein, J., Rungta, R., Saladi, K., Schelten, A., Silva, R., Smith, E.~M., Subramanian, R., Tan, X.~E., Tang, B., Taylor, R., Williams, A., Kuan, J.~X., Xu, P., Yan, Z., Zarov, I., Zhang, Y., Fan, A., Kambadur, M., Narang, S., Rodriguez, A., Stojnic, R., Edunov, S., and Scialom, T.
\newblock Llama 2: Open foundation and fine-tuned chat models, 2023{\natexlab{b}}.

\bibitem[Viazovska(2017)]{maryna}
Viazovska, M.
\newblock The sphere packing problem in dimension $8$.
\newblock \emph{Annals of Mathematics}, 185\penalty0 (3), May 2017.
\newblock ISSN 0003-486X.
\newblock \doi{10.4007/annals.2017.185.3.7}.
\newblock URL \url{http://dx.doi.org/10.4007/annals.2017.185.3.7}.

\bibitem[Xi et~al.(2023)Xi, Li, Chen, and Zhu]{xi20234bit}
Xi, H., Li, C., Chen, J., and Zhu, J.
\newblock Training transformers with 4-bit integers.
\newblock In \emph{Thirty-seventh Conference on Neural Information Processing Systems}, 2023.
\newblock URL \url{https://openreview.net/forum?id=H9hWlfMT6O}.

\end{thebibliography}
\bibliographystyle{icml2024}

\newpage
\appendix
\onecolumn

\section{Concentration Inequalities for the Randomized Hadamard Transform and Fast Fourier Transform}

\subsection{Incoherence Processing with the Randomized Hadamard Transform}

\begin{lemma}
For any non-negative real number $n$,
\[
    \frac{1}{B\left( n, 1/2 \right)} 
    \int_{-1}^{+1} (1 - x^2)^{n-1} \cdot \exp(tx) \; dx
    \le
    \exp\left( \frac{t^2}{4n+2} \right).
\]
\label{lemmaIntegral}
\end{lemma}
\begin{proof}
We start with the following ``standard'' integral. For non-negative integer $m$ and real $n > 0$,
\begin{align*}
    \int_{-1}^{+1} x^{2m} (1 - x^2)^{n-1} \; dx
    =
    B\left( m + \frac{1}{2}, n \right)
    =
    \frac{
        \Gamma\left( m + \frac{1}{2} \right)
        \Gamma\left( n \right)
    }{
        \Gamma\left( m + n + \frac{1}{2} \right)
    }.
\end{align*}
This means that
\begin{align*}
    \frac{1}{B\left( \frac{1}{2}, n \right)} \int_{-1}^{+1} x^{2m} (1 - x^2)^{n-1} \; dx
    &=
    \frac{
        B\left( m + \frac{1}{2}, n \right)
    }{
        B\left( \frac{1}{2}, n \right)
    }
    \\&=
    \frac{
        \Gamma\left( m + \frac{1}{2} \right)
        \Gamma\left( n \right)
    }{
        \Gamma\left( m + n + \frac{1}{2} \right)
    }
    \cdot
    \frac{
        \Gamma\left( n + \frac{1}{2} \right)
    }{
        \Gamma\left( \frac{1}{2} \right)
        \Gamma\left( n \right)
    }
    \\&=
    \frac{
        \Gamma\left( m + \frac{1}{2} \right)
        \Gamma\left( n + \frac{1}{2} \right)
    }{
        \sqrt{\pi}
        \cdot 
        \Gamma\left( m + n + \frac{1}{2} \right)
    }.
\end{align*}
Applying the Legendre duplication formula, for integer $m$,
\[
    \Gamma\left(m + \frac{1}{2} \right) = \frac{(2m)! \sqrt{\pi}}{4^m m!},
\]
then
\begin{align*}
    \frac{1}{B\left( \frac{1}{2}, n \right)} \int_{-1}^{+1} x^{2m} (1 - x^2)^{n-1} \; dx
    &=
    \frac{(2m)! \sqrt{\pi}}{4^m m!}
    \cdot
    \frac{(2n)! \sqrt{\pi}}{4^n n!}
    \cdot
    \frac{1}{\sqrt{\pi}}
    \cdot
    \frac{4^{m+n} (m+n)!}{(2m+2n)! \sqrt{\pi}}
    \\&=
    \frac{(2m)!(2n)!(m+n)!}{m!n!(2m+2n)!}.
\end{align*}
In particular, this means that
\begin{align*}
    \frac{1}{B\left( \frac{1}{2}, n \right)} 
    \int_{-1}^{+1} \exp(tx) (1 - x^2)^{n-1} \; dx
    &=
    \sum_{m=0}^{\infty}
    \frac{t^{2m}}{(2m)!}
    \cdot
    \frac{1}{B\left( \frac{1}{2}, n \right)} 
    \int_{-1}^{+1} x^{2m} (1 - x^2)^{n-1} \; dx
    \\&=
    \sum_{m=0}^{\infty}
    \frac{t^{2m}}{(2m)!}
    \cdot
    \frac{(2m)!(2n)!(m+n)!}{m!n!(2m+2n)!}
    \\&=
    \sum_{m=0}^{\infty}
    \frac{t^{2m}}{m!}
    \cdot
    \frac{(2n)!(m+n)!}{n!(2m+2n)!}
    \\&=
    \sum_{m=0}^{\infty}
    \frac{t^{2m}}{m!}
    \cdot
    \prod_{k=1}^m
    \frac{k+n}{(2k+2n)(2k+2n-1)}
    \\&=
    \sum_{m=0}^{\infty}
    \frac{t^{2m}}{m!}
    \cdot
    \frac{1}{2^m}
    \prod_{k=1}^m
    \frac{1}{2k+2n-1}
    \\&\le
    \sum_{m=0}^{\infty}
    \frac{t^{2m}}{m!}
    \cdot
    \frac{1}{2^m}
    \left( \frac{1}{2n+1} \right)^m
    \\&=
    \sum_{m=0}^{\infty}
    \frac{1}{m!}
    \left( \frac{t^2}{4n+2} \right)^m
    \\&=
    \exp\left( \frac{t^2}{4n+2} \right).
\end{align*}
This proves the lemma
\[
    \frac{1}{B\left( n, 1/2 \right)} 
    \int_{-1}^{+1} (1 - x^2)^{n-1} \cdot \exp(tx) \; dx
    \le
    \exp\left( \frac{t^2}{4n+2} \right).
\]
\end{proof}

\begin{lemma}
Call $U \in \R^{nd \times nd}$ an $(n,d)$-block orthohadamard matrix if it has the following properties: (1) $U$ is a orthogonal matrix, and (2) each aligned $d \times d$ block of $U$ is $1/\sqrt{n}$ times an orthogonal matrix.
This generalizes the notion of Hadamard matrices.
Let $S \in \R^{nd \times nd}$ be a random block diagonal matrix, where each $d \times d$ block of the diagonal is sampled independently and uniformly from the set of (possibly special) orthogonal matrices.
Then we call multiplication by $US$ a \emph{randomized orthohadamard transform}, and observe that it has the following nice property.
Let $x \in \mathbb{R}^{nd}$ be any fixed vector, and let $b \in \mathbb{R}^{nd}$ be a fixed vector that is \emph{sparse} in the sense that it is supported only on a single $d$-sized aligned block (i.e. all but one of the $n$ blocks are zero). Then
\[
    \mathbf{P}\left( \Abs{b^T U S x} \ge a \right)
    \le
    2 \exp\left( -\frac{a^2 n d}{2 \norm{b}^2 \norm{x}^2} \right).
\]
\label{lemmaOrthoHadamard}
\end{lemma}
\begin{proof}
If we let the $i$th block of $x$ be $x_i \in \R^d$ and let the $i$th block of $S^T U^T b^T$ be $v_i$, then the $v_i$ will be independent and uniformly distributed on the sphere in $d$ dimensional space of radius $\norm{b} / \sqrt{n}$, and so $v_i^T x_i = \norm{b} \norm{x_i} n^{-1/2} z_i$, where the $z_i$ are all independent and distributed according to an entry of a random point on the unit sphere in $d$ dimensional space. Observe that this means that
\[
    \mathbf{P}(z_i) \propto (1 - z_i^2)^{\frac{d-1}{2} - 1}.
\]
So,
\begin{align*}
    \Exv{ \exp\left( t b^T U S x \right) }
    &=
    \Exv{ \exp\left( t \sum_{i=1}^n \norm{b} \norm{x_i} n^{-1/2} z_i \right) }
    \\&=
    \prod_{i=1}^n \Exv{ \exp\left( t \norm{b} \norm{x_i} n^{-1/2} z_i \right) }
    \\&\le
    \prod_{i=1}^n \Exv{ \exp\left( \frac{1}{4\cdot \frac{d-1}{2} + 2} \right) \left( t \norm{b} \norm{x_i} n^{-1/2} \right)^2 }
    \\&=
    \prod_{i=1}^n \Exv{ \exp\left( \frac{t^2 \norm{b}^2 \norm{x_i}^2}{2nd} \right) }
    \\&=
    \Exv{ \exp\left( \frac{t^2 \norm{b}^2 \norm{x}^2}{2nd} \right) },
\end{align*}
where the the last line follows from Lemma~\ref{lemmaIntegral}.
It follows from the standard application of Markov's inequality that for any $a > 0$,
\[
    \mathbf{P}\left( \Abs{b^T U S x} \ge a \right)
    \le
    2 \exp\left( -\frac{a^2 n d}{2 \norm{b}^2 \norm{x}^2} \right).
\]
This is what we wanted to show.
\end{proof}

\begin{lemma}
\label{lemma:H_hadfftconcent}
Let $H \in \mathbb{R}^{n \times n}$ be an orthogonal scaled Hadamard matrix or $F \in \mathbb{R}^{n \times n}$ be an orthogonal FFT matrix (the FFT understood as operating on a real vector space). Let $S \in \mathbb{R}^{n \times n}$ be a random diagonal matrix with diagonal elements supported on $\R^n$, and let $P \in \mathbb{R}^{n \times n}$ be a random 2-block-diagonal matrix with $2 \times 2$ diagonal blocks supported on $\operatorname{SO}(2)$ (we can also think of this as acting like a diagonal complex matrix with each diagonal element a random complex number of absolute value $1$).
Let $U \in \mathbb{R}^{n \times n}$ be any fixed orthogonal matrix.
Then, for any $\epsilon > 0$,
\[
    \Prob{ \max_{i,j} \Abs{e_i^T H S U e_j} \ge \sqrt{
        \frac{2}{nd} \log\left( \frac{2 n^2}{\epsilon} \right)
    }} \le \epsilon
\]
and
\[
    \Prob{ \max_{i,j} \Abs{e_i^T F P U e_j} \ge \sqrt{
        \frac{2}{nd} \log\left( \frac{2 n^2}{\epsilon} \right)
    }} \le \epsilon.
\]
That is, with probability at least $1 - \epsilon$, multiplying by either $HS$ or $FP$ makes the resulting orthogonal matrix $\mu$-incoherent, where
\[
    \mu_H = \sqrt{
        2 \log\left( \frac{2 n^2}{\epsilon} \right)
    }.
\]
\end{lemma}
\begin{proof}
Setting $b = e_i$ and $x = U e_j$ in Lemma~\ref{lemmaOrthoHadamard},
\[
    \mathbf{P}\left( \Abs{e_i^T H S U e_j} \ge a \right)
    \le
    2 \exp\left( -\frac{a^2 n d}{2} \right).
\]
By the union bound,
\[
    \mathbf{P}\left( \max_{i,j} \Abs{e_i^T H S U e_j} \ge a \right)
    \le
    2n^2 \exp\left( -\frac{a^2 n d}{2} \right).
\]
Setting
\[
    a^2 = \frac{2}{nd} \log\left( \frac{2 n^2}{\epsilon} \right)
\]
proves the lemma. The FFT case is identical.
\end{proof}

\begin{lemma}
\label{lemma:W_hadfftconcent}
Let $H_L \in \mathbb{R}^{m \times m}$ be an orthogonal scaled Hadamard matrix or $F \in \mathbb{R}^{m \times m}$ be an orthogonal FFT matrix (the FFT understood as operating on a real vector space). Let $S_L \in \mathbb{R}^{m \times m}$ be a random diagonal matrix with diagonal elements supported on $\R^m$, and let $P \in \mathbb{R}^{m \times m}$ be a random 2-block-diagonal matrix with $2 \times 2$ diagonal blocks supported on $\operatorname{SO}(2)$. Let $H_R \in \mathbb{R}^{n \times n}$, $F_R$, $S_R$, and $P_R$ be defined analogously over $n$-dimensional space.
Let $W \in \mathbb{R}^{m \times n}$ be any fixed matrix.
Then, for any $\epsilon > 0$,
\[
    \mathbf{P}\left( \max_{i,j} \Abs{e_i^T H_L S_L W S_R^T H_R^T e_j} \ge \norm{W}_F \sqrt{ \frac{4}{mn} } \log\left( \frac{4mn}{\epsilon} \right) \right)
    \le
    \epsilon.
\]
and
\[
    \mathbf{P}\left( \max_{i,j} \Abs{e_i^T F_L P_L W P_R^T F_R^T e_j} \ge \norm{W}_F \sqrt{ \frac{4}{mn} } \log\left( \frac{4mn}{\epsilon} \right) \right)
    \le
    \epsilon.
\]
That is, with probability at least $1 - \epsilon$, multiplying on both sides by a randomized Hadamard transform or a randomized FFT yields a weight matrix that is $\mu_W$-incoherent, where
\[
    \mu_W = 2 \log\left( \frac{4 m n}{\epsilon} \right).
\]
\end{lemma}
\begin{proof}
From Lemma~\ref{lemmaOrthoHadamard},
\[
    \mathbf{P}\left( \Abs{b^T U S x} \ge \norm{b} \norm{x} \sqrt{ \frac{2}{n} \log\left( \frac{4mn}{\epsilon} \right) } \right)
    \le
    \frac{\epsilon}{2mn}.
\]
By applying this once on each side to the rows and columns respectively, and union bounding over the $mn$ entries, we get
\[
    \mathbf{P}\left( \Abs{e_i^T H_L S_L W S_R^T H_R^T e_j} \ge \norm{W}_F \sqrt{ \frac{4}{mn} } \log\left( \frac{4mn}{\epsilon} \right) \right)
    \le
    \epsilon.
\]
The proof in the FFT case is identical.
\end{proof}

\lemmahadincoh*
\begin{proof}
The incoherence of $H$ follows from the application of Lemma~\ref{lemma:H_hadfftconcent}.
The incoherence of $W$ follows from the application of Lemma~\ref{lemma:W_hadfftconcent}.
\end{proof}

\subsection{Incoherence Processing with the Randomized Fast Fourier Transform (RFFT)}
\label{sec:fft}
\begin{algorithm}[t]
\caption{Incoherence Processing with RFFT (IP-RFFT)}
\begin{algorithmic}
\INPUT $W \in \mathbb{R}^{m \times n}, H \in \mathbb{R}^{n \times n}$
\STATE Sample phase vectors 
\\ \hspace{2.5em} $\theta_V \sim \mathcal{U}[0, 2\pi]^{n/2}$, $\theta_U \sim \mathcal{U}[0, 2\pi]^{m/2}$
\\ \hspace{2.5em} $S_V = \cos(\theta_V) + i \sin(\theta_V)$,  $S_U = \cos(\theta_U) + i \sin(\theta_U)$
\STATE $\hat W \gets \texttt{FFT}(diag(S_U)\texttt{FFT}(diag(S_V) W^T)^T)$ where 
\\ \hspace{2.5em} $\texttt{FFT}$ is the Fast Fourier transform (Sec. \ref{sec:fft})
\STATE $\hat H \gets \texttt{FFT}(diag(S_V) \texttt{FFT}(diag(S_V) H)^T)$
\OUTPUT $\hat W, \hat H, S_U, S_V$
\end{algorithmic}
\label{alg:RFFT}
\end{algorithm}

 Here we described the Randomized Fast Fourier Transform (RFFT), $x \to V S x$ where $V \in \mathbb{C}^{n/2 \times n/2}$ is the discrete Fourier transform matrix, $S \in \mathbb{C}^{n/2}$ is a random complex phase vector, and $x \in \mathbb{R}^n$. 
 The discrete Fourier transform can be computed in $O(n \log n)$ time via the fast Fourier transform.
 Here it is understood that the FFT operates over the reals, in that a vector $x \in \mathbb{R}^n$ is mapped to a complex representation $\mathbb{C}^{n/2}$, the RFFT is performed, and the resulting vector mapped back to $\mathbb{R}^n$.
 Here the mapping simply represents reshaping real-valued $x$ into dimension $(n/2,2)$, and interpreting the corresponding 2-tuples as a complex number.

Incoherence processing via the RFFT achieves similar theoretical guarantees as the RHT, see Lemmas~\ref{lemma:H_hadfftconcent} and~\ref{lemma:W_hadfftconcent}.
Ultimately the choice of the orthogonal transformation is up to the user.
A Fourier transform works almost as well as a Hamard transform in practice (Table~\ref{tab:ablate_incoherence}), so if a fast Hadamard implementation is not available, the FFT is a good option.

\section{Block LDLQ}
\label{sec:blockldlqtheory}

\begin{lemma}
\label{lemma:C1}
Let $H \in \mathbb{R}^{nd \times nd}$ be a positive definite matrix with $d$-block LDL decomposition $H = L^T D L$. Then
\[
    \trace{D} \le \trace{H^{1/2}} \cdot \norm{ H^{1/2} \odot M_D }_2,
\]
where $M_D = I \otimes \mathbf{1}_{d \times d}$ is the block diagonal mask.
If, in addition, $H$ is $\mu$-incoherent in the sense that its matrix of eigenvectors $U$ has
\[
    \| U_{ij} \| \le \frac{\mu}{\sqrt{nd}},
\]
then
\[
    \trace{D} \le \frac{\mu^2}{n} \trace{H^{1/2}}^2.
\]
\end{lemma}
\begin{proof}
Consider the optimization problem
\begin{align*}
\mbox{minimize: } & \trace{R^T H R}
\\\mbox{subject to: } & R \text{ unit block lower diagonal}.
\end{align*}
Observe that the derivative of the loss is
\[
    \nabla f(R) = HR.
\]
If $R = L^{-1}$, then $HR = L^T D$.
But this must be a block upper triangular matrix, because it's the product of a unit upper triangular matrix ($L^T$) and a block diagonal matrix $D$. It follows that $\nabla f(L^{-1})$ is zero in all the directions in which we could move $R$, since $R$ only varies in the strictly lower triangular directions. Therefore, $R = L^{-1}$ is the solution to this optimization problem, and for any $R$, $\nabla f(R) \ge \nabla f(L^{-1}) = \trace{D}$.

Now, let $M$ denote the strictly block lower triangular mask, and observe that $M + M^T + M_D = \mathbf{1}_{nd \times nd}$.
Set $\alpha = \norm{ H^{1/2} \odot M_D }_2^{-1}$, and consider $R = \left( I + \alpha M \odot H^{1/2} \right)^{-1}$.
Observe that
\begin{align*}
    \left( I + \alpha M \odot H^{1/2} \right)^T \left( I + \alpha M \odot H^{1/2} \right)
    &=
    I + \alpha M \odot H^{1/2} + \alpha M^T \odot H^{1/2} + \alpha^2 (M^T \odot H^{1/2}) (M \odot H^{1/2})
    \\&\succeq
    I + \alpha (M + M^T) \odot H^{1/2}
    \\&\succeq
    \alpha M_D \odot H^{1/2} + \alpha (M + M^T) \odot H^{1/2}
    \\&\succeq
    \alpha H^{1/2}.
\end{align*}
It follows by inverting both sides that $R R^T \preceq  \alpha^{-1} H^{-1/2}$.

So, for this $R$,
\[
    \trace{R^T H R}
    =
    \trace{H R R^T}
    \le
    \alpha^{-1} \trace{H^{1/2}}.
\]
This proves the first part of the lemma.
For the second part, observe that
\begin{align*}
    \norm{ H^{1/2} \odot M_D }_2
    &\le
    \sum_{i=1}^{nd} \lambda_i^{1/2} \norm{ (u_i u_i^T) \odot M_D }_2
    \\&\le
    \frac{\mu^2}{n} \trace{H^{1/2}}.
\end{align*}
This proves the lemma.
\end{proof}

\thmLDLQ*
\begin{proof}
First recall that from the description of block LDLQ,
\[
\hat W_k = \mathbf{Q}(W_k + (W_{:(k-1)} - \hat W_{:(k-1)})\mathbf{A}_k).
\]
We can also write this in matrix form in terms of the matrix $\mathbf{L}_k$ as
\[
\hat W = \mathbf{Q}(W + (W - \hat W)( \mathbf{L}^T - I )).
\]
Here, $\mathbf{Q}$ is interpreted as operating independently block-wise. Let $\eta$ denote the quantization error
\[
    \eta = (W + (W - \hat W)( \mathbf{L}^T - I )) -  \mathbf{Q}(W + (W - \hat W)( \mathbf{L}^T - I )).
\]
Then
\[
\hat W = (W + (W - \hat W)( \mathbf{L}^T - I )) - \eta,
\]
which simplifies to
\[
(W - \hat W) \mathbf{L}^T = \eta.
\]
This means that
\[
\mathbf{E}[\operatorname{tr}((\hat W - W)H(\hat W - W)^T)]
=
\mathbf{E}[\operatorname{tr}((\hat W - W)\mathbf{L}^T \mathbf{D} \mathbf{L} (\hat W - W)^T)]
=
\mathbf{E}[\operatorname{tr}(\eta^T \mathbf{D} \eta)].
\]
But by assumption, $\mathbf{E}[\eta \eta^T] \preceq m \sigma^2 I$ (since each block is just an independent application of $\mathbf{Q}$ and we sum over $m$ rows), so
\[
\mathbf{E}[\operatorname{tr}((\hat W - W)H(\hat W - W)^T)]
\le
m \sigma^2 \mathbf{E}[\operatorname{tr}(\mathbf{D})].
\]
Combining this with the result of Lemma~\ref{lemma:C1} proves the theorem.
\end{proof}

\section{E8P details}

\subsection{Constructing $S$}
\label{sec:e8ps}

We use the following 29 elements of $\hat D_8$ with norm squared 12 to pad $S$ to 256 entries.

\begin{verbatim}
   ([3, 1, 1, 1, 3, 3, 3, 3] [1, 3, 1, 1, 3, 3, 3, 3] [1, 1, 3, 1, 3, 3, 3, 3]
    [1, 1, 1, 3, 3, 3, 3, 3] [3, 3, 3, 1, 3, 3, 1, 1] [3, 3, 3, 1, 3, 1, 3, 1]
    [3, 3, 3, 1, 1, 3, 3, 1] [3, 3, 3, 1, 3, 1, 1, 3] [3, 3, 3, 1, 1, 3, 1, 3]
    [3, 3, 3, 1, 1, 1, 3, 3] [3, 3, 1, 3, 3, 3, 1, 1] [3, 3, 1, 3, 3, 1, 3, 1]
    [3, 3, 1, 3, 1, 3, 3, 1] [3, 3, 1, 3, 3, 1, 1, 3] [3, 3, 1, 3, 1, 3, 1, 3]
    [3, 3, 1, 3, 1, 1, 3, 3] [3, 1, 3, 3, 3, 3, 1, 1] [3, 1, 3, 3, 3, 1, 3, 1]
    [3, 1, 3, 3, 1, 3, 3, 1] [3, 1, 3, 3, 3, 1, 1, 3] [3, 1, 3, 3, 1, 3, 1, 3]
    [1, 3, 3, 3, 1, 1, 3, 3] [1, 3, 3, 3, 3, 3, 1, 1] [1, 3, 3, 3, 3, 1, 3, 1]
    [1, 3, 3, 3, 1, 3, 3, 1] [1, 3, 3, 3, 3, 1, 1, 3] [1, 3, 3, 3, 1, 3, 1, 3]
    [1, 1, 3, 3, 1, 3, 3, 3] [3, 3, 1, 1, 3, 3, 3, 1] ) / 2
\end{verbatim}

\subsection{Example Decoding with E8P}
Here, we give an example of decoding with E8P.
In this example, the first 8 bits of the codeword encode the entry in $S$, the next 7 bits encode the 7 right sign flips, and the last bit encodes whether or not we shift by $\frac{1}{4}$.
Let the codeword be \texttt{0001010110010111}.
The first 8 bits \texttt{00010101} = 21 would indicate that we start with the 21st entry in $S$.
In this example, let that be the vector
$$s = \left\{\frac{1}{2}, \frac{1}{2}, \frac{1}{2}, \frac{3}{2}, \frac{1}{2}, \frac{1}{2}, \frac{1}{2}, \frac{1}{2}\right\},$$

which is not in $\hat{D_8}$.
Thus, $s$ requires an odd number of sign flips to get into $\hat{D_8}$.
Then, the next 7 bits \texttt{1001011} would indicate that we need to negate the 1st, 2nd, 4th, and 7th from right bits.
Since we need an odd number of sign flips, the 8th from right bit is also a sign flip.
The sign-decoded vector is then
$$\left\{-\frac{1}{2}, -\frac{1}{2}, \frac{1}{2}, \frac{3}{2}, -\frac{1}{2}, \frac{1}{2}, -\frac{1}{2}, -\frac{1}{2}\right\},$$

which we can verify is in $E_8$.
Finally, the last bit \texttt{1} indicates that we need to add $\frac{1}{4}$, so the final decoded vector is
$$\left\{-\frac{1}{4}, -\frac{3}{4}, \frac{3}{4}, \frac{7}{4}, -\frac{1}{4}, \frac{3}{4}, -\frac{1}{4}, -\frac{1}{4}\right\},$$

which is in $E_8 + \frac{1}{4}$ as desired.

For posterity, we include a copy of our CUDA kernel for matrix-vector multiplication with E8P. 
This kernel was designed for NVIDIA Ampere and newer GPUs.
The same kernel can be found at \url{https://github.com/Cornell-RelaxML/quip-sharp/blob/main/quiptools/quiptools_e8p_gemv.cu}.

\begin{lstlisting}[language=C++,frame=single,basicstyle=\small\ttfamily]
__global__ static void
decode_matvec_e8p_kernel(
    float *__restrict__ output,
    const uint2 *__restrict__ input,
    const uint2 *__restrict__ weights_compressed,
    const uint32_t *__restrict__ codebook_abs,
    int N,
    int K
) {
    int warpId = threadIdx.y;
    int laneId = threadIdx.x;

    for (int iin = blockIdx.x; iin < (N >> 4); iin += gridDim.x) {

        float z0 = 0.0;
        float z1 = 0.0;
        float z2 = 0.0;
        float z3 = 0.0;

        for (int iik = warpId; iik < (K >> 6); iik += 32) {
            uint2 w_compr = weights_compressed[laneId + 32*iik + (K >> 1)*iin];
            uint32_t a = w_compr.x;
            uint32_t b = w_compr.y;

            uint32_t s = b;
            s = s ^ (s >> 4);
            s = s ^ (s >> 8);
            s = s ^ (s >> 16);
            uint32_t sb = (s & 15);
            s = b ^ sb;
            sb = sb | (sb << 16);

            uint32_t input_to_warp = ((const uint32_t*)(&input[16*iik]))[laneId];
            uint32_t shifted_laneId = (laneId & 3) << 3;

            // BLOCK 01
            {
            uint32_t x = codebook_abs[(a >> 0) & 255];
            x = x ^ ((s & 0x11111111) * 14);

            uint32_t o = BASE_OFFSET | ((sb & 0x00010001) << 4);

            uint32_t w00 = add_as_half2(mask_lop3(x << 4, XMASK, WMASK), o);
            uint32_t w01 = add_as_half2(mask_lop3(x << 0, XMASK, WMASK), o);
            uint32_t w02 = add_as_half2(mask_lop3(x >> 4, XMASK, WMASK), o);
            uint32_t w03 = add_as_half2(mask_lop3(x >> 8, XMASK, WMASK), o);

            x = codebook_abs[(a >> 8) & 255];
            x = x ^ ((s & 0x22222222) * 7);

            o = BASE_OFFSET | ((sb & 0x00020002) << 3);
            
            uint32_t w10 = add_as_half2(mask_lop3(x << 4, XMASK, WMASK), o);
            uint32_t w11 = add_as_half2(mask_lop3(x << 0, XMASK, WMASK), o);
            uint32_t w12 = add_as_half2(mask_lop3(x >> 4, XMASK, WMASK), o);
            uint32_t w13 = add_as_half2(mask_lop3(x >> 8, XMASK, WMASK), o);

            uint32_t x_in0 = __shfl_sync(FULL_MASK, input_to_warp, shifted_laneId | 0);
            uint32_t x_in1 = __shfl_sync(FULL_MASK, input_to_warp, shifted_laneId | 1);

            asm(
                "mma.sync.aligned.m16n8k16.row.col.f32.f16.f16.f32"
                " { %0, %1, %2, %3 },"
                " { %4, %5, %6, %7 },"
                " { %8, %9 },"
                " { %0, %1, %2, %3 };"
                : "+f"(z0), "+f"(z1), "+f"(z2), "+f"(z3)
                : "r"(w00), "r"(w10), "r"(w01),  "r"(w11),
                  "r"(x_in0), "r"(x_in1)
            );

            x_in0 = __shfl_sync(FULL_MASK, input_to_warp, shifted_laneId | 2);
            x_in1 = __shfl_sync(FULL_MASK, input_to_warp, shifted_laneId | 3);

            asm(
                "mma.sync.aligned.m16n8k16.row.col.f32.f16.f16.f32"
                " { %0, %1, %2, %3 },"
                " { %4, %5, %6, %7 },"
                " { %8, %9 },"
                " { %0, %1, %2, %3 };"
                : "+f"(z0), "+f"(z1), "+f"(z2), "+f"(z3)
                : "r"(w02), "r"(w12), "r"(w03), "r"(w13),
                  "r"(x_in0), "r"(x_in1)
            );
            }
            // BLOCK 23 
            {
            uint32_t x = codebook_abs[(a >> 16) & 255];
            s = s >> 2;
            x = x ^ ((s & 0x11111111) * 14);

            uint32_t o = BASE_OFFSET | ((sb & 0x00040004) << 2);
            
            uint32_t w00 = add_as_half2(mask_lop3(x << 4, XMASK, WMASK), o);
            uint32_t w01 = add_as_half2(mask_lop3(x << 0, XMASK, WMASK), o);
            uint32_t w02 = add_as_half2(mask_lop3(x >> 4, XMASK, WMASK), o);
            uint32_t w03 = add_as_half2(mask_lop3(x >> 8, XMASK, WMASK), o);

            x = codebook_abs[(a >> 24) & 255];
            x = x ^ ((s & 0x22222222) * 7);

            o = BASE_OFFSET | ((sb & 0x00080008) << 1); 

            uint32_t w10 = add_as_half2(mask_lop3(x << 4, XMASK, WMASK), o);
            uint32_t w11 = add_as_half2(mask_lop3(x << 0, XMASK, WMASK), o);
            uint32_t w12 = add_as_half2(mask_lop3(x >> 4, XMASK, WMASK), o);
            uint32_t w13 = add_as_half2(mask_lop3(x >> 8, XMASK, WMASK), o);

            uint32_t x_in0 = __shfl_sync(FULL_MASK, input_to_warp, shifted_laneId | 4);
            uint32_t x_in1 = __shfl_sync(FULL_MASK, input_to_warp, shifted_laneId | 5);

            asm(
                "mma.sync.aligned.m16n8k16.row.col.f32.f16.f16.f32"
                " { %0, %1, %2, %3 },"
                " { %4, %5, %6, %7 },"
                " { %8, %9 },"
                " { %0, %1, %2, %3 };"
                : "+f"(z0), "+f"(z1), "+f"(z2), "+f"(z3)
                : "r"(w00), "r"(w10), "r"(w01), "r"(w11),
                  "r"(x_in0), "r"(x_in1)
            );

            x_in0 = __shfl_sync(FULL_MASK, input_to_warp, shifted_laneId | 6);
            x_in1 = __shfl_sync(FULL_MASK, input_to_warp, shifted_laneId | 7);

            asm(
                "mma.sync.aligned.m16n8k16.row.col.f32.f16.f16.f32"
                " { %0, %1, %2, %3 },"
                " { %4, %5, %6, %7 },"
                " { %8, %9 },"
                " { %0, %1, %2, %3 };"
                : "+f"(z0), "+f"(z1), "+f"(z2), "+f"(z3)
                : "r"(w02), "r"(w12), "r"(w03), "r"(w13),
                  "r"(x_in0), "r"(x_in1)
            );
            }
        }

        if ((laneId & 1) == 0) {
            atomicAdd(output + (iin << 4) + (laneId >> 1), (laneId & 2) ? z2 : z0);
        }
    }
}
\end{lstlisting}

\subsection{Why not K-Means?}
\label{sec:kmeans}

A significant motivating factor behind E8P is that post-incoherence processing, entries of $W$ are approximately Gaussian distributed.
However, E8P is uniformly distributed, which raises the question: why not use a K-means based codebook?
K-means based codebooks offer strong theoretical performance but have a few issues.
First, it is difficult to enforce symmetry in a ``learned'' K-means codebook.
This is crucial to be able to have a compressible codebook.
If we force sign symmetry by learning cluster centers on only the positive orthant of a $n$-dimensional Gaussian, we can get around this but sacrifice accuracy at the axis region.
Second, using K-means requires storing a codebook in fp16, whereas the entries of E8P can be stored as 4 bit integers.
This means that during inference, the source codebook for a 8 dimension K-means codebook will be 4 times larger than the source codebook of E8P, running the risk of a cache eviction.
Finally, we observe that empirically, E8P actually \textit{outperforms} K-means, which is somewhat interesting and suggests that allocating more information to the edge of the distribution, even after incoherence processing, is useful.

\subsection{E8P vs. Other Codebook Constructions}

Below, we compare the quality of quantized models from E8P vs. other codebooks.
The $D_4$ lattice achieves the 4 dimensional kissing number but has lower dimensionality than $E_8$, giving poorer shaping.
An 8-dimensional K-Means codebook has similar shaping as E8P but worse packing density.
Although RHT-transformed weights are approximately Gaussian and not Uniform, we find that a Uniform codebook (E8P) performs better than a Gaussian codebook (K-means)

\begin{table}[h]
\caption{E8P and $E_8$-based codebooks outperform other codebooks based on lower-dimensional lattices or different distributions. Numbers without fine-tuning.}
\centering
\sc
\begin{tabular}{@{}cccccc@{}}
\toprule
Model          & Codebook      & Codebook Dim. & Bits       & Wiki2 PPL (ctx 4096) & C4 PPL (ctx 4096, c4_new) \\ \midrule
2-70B          & FP16          & 1             & 16         & 3.120                & 5.533                     \\
\textbf{2-70B} & \textbf{E8P}  & \textbf{8}    & \textbf{2} & \textbf{4.156}       & \textbf{6.535}            \\
2-70B          & $E_8$ lattice & 8             & 2.37       & 3.702                & 6.082                     \\
2-70B          & $D_4$ lattice & 4             & 2          & 4.408                & 6.797                     \\
2-70B          & $D_4$ lattice & 4             & 2.21       & 3.970                & 6.332                     \\
2-70B          & K-Means       & 8             & 2          & 4.452                & 6.925                     \\ \bottomrule
\end{tabular}
\end{table}

\section{Fine-Tuning During Quantization}
\label{sec:ft}
In Algorithm~\ref{alg:FT} we describe our fine tuning procedure for \qs.

\begin{algorithm}
\caption{\qss with Fine-Tuning}
\label{alg:FT}
\begin{algorithmic}
\REQUIRE Unquantized Model $M$, Development Set $\mathcal{D}$, Quantization Order $O$,
\ENSURE Quantized Model $M$
\STATE $X \gets M_{\mbox{embedding}}(\mathcal{D})$
\STATE $C \gets M(\mathcal{D})_{\mbox{logits}}$
\FOR {Decoder Block $D \in M$}
\STATE $Y \gets D(X)$
\STATE $X_{train}, Y_{train}, X_{valid}, Y_{valid} \gets \mbox{split}(X, Y)$
\FOR {Linear Layer $L \in D$ in order specified by $O$}
\STATE $\hat L \gets \mbox{\qs-NoFT}(L)$
\STATE Disable gradients for the weight matrix (but not $S_U, S_V)$ of $\hat L$.
\STATE Optimize $D$ to minimize $\mbox{MSE}(D(X_{train}), Y_{train})$ using $X_{valid}, Y_{valid}$ for early stopping.
\ENDFOR
\STATE $X \gets Y$
\ENDFOR

\COMMENT {At this point, the learnable parameters in $M$ are the layernorms, all $S_U$ and $S_V$, and the language model head.}
\STATE $\mathcal{D}_{train}, C_{train}, \mathcal{D}_{valid}, C_{valid} \gets \mbox{split}(\mathcal{D}, C)$
\STATE Optimize $M$ to minimize $\mbox{CrossEntropy}(M(\mathcal{D}_{train}), C_{train})$ using $\mathcal{D}_{valid}, C_{valid}$ for early stopping.
\end{algorithmic}
\end{algorithm}

\section{Additional Results}

\subsection{\qss vs. OmniQuant with Grouping}
\label{sec:omniqg}
\begin{table}[h]
\caption{\qss outperforms OmniQuant even with grouping. More numbers can be found in the OmniQuant paper.}
\sc
\centering
\begin{tabular}{@{}ccccc@{}}
\toprule
Model & Method          & Effective Bits & Wiki2 PPL (ctx 2048) & C4 PPL (ctx 2048) \\ \midrule
2-70B & FP16            & 16             & 3.32                 & 5.52              \\
2-70B & QuIP\#          & 2              & 4.16                 & 6.12              \\
2-70B & OmniQ W2A16     & 2              & 7.81                 & 12.28             \\
2-70B & OmniQ W2A16 g64 & 2.25           & 6.11                 & 7.88              \\
2-70B & OmniQ W3A16     & 3              & 3.92                 & 6.06              \\ \bottomrule
\end{tabular}
\end{table}

\subsection{\qss on Mixtral 8x7B \cite{mixtral} and Falcon 180B \cite{falcon}}

\begin{table}[H]
\caption{2 bit \qss without fine-tuning on Mixtral 8x7B, a mixture of experts (MoE), and Falcon 180B, a non-Llama-based model. \qss scales to different architectures without issue.}
\centering
\centering
\sc
\begin{tabular}{@{}ccccccccc@{}}
\toprule
Model        & Bits & Wiki2 & C4   & ArcC & ArcE & BoolQ & PiQA & Wino \\ \midrule
\rowcolor[HTML]{D9D9D9} 
Mixtral-8x7B & 16   & 3.45  & 6.85 & 0.56 & 0.74 & 0.85  & 0.84 & 0.75 \\
\rowcolor[HTML]{D9D9D9} 
Mixtral-8x7B & 2    & 4.69  & 8.25 & 0.49 & 0.68 & 0.81  & 0.80 & 0.73 \\
\rowcolor[HTML]{FFFFFF} 
Falcon-180B  & 16   & 3.30  & 6.31 & 0.61 & 0.82 & 0.87  & 0.85 & 0.81 \\
\rowcolor[HTML]{FFFFFF} 
Falcon-180B  & 2    & 4.18  & 7.06 & 0.58 & 0.81 & 0.84  & 0.84 & 0.81 \\ \bottomrule
\end{tabular}
\end{table}

\subsection{Zeroshot performance for ablation on lattice codebooks and fine-tuning}

\begin{table}[H]
\caption{Ablation on lattice codebooks and fine-tuning. \qss no FT and $E_8$ uses the RHT to perform incoherence processing but does not use lattice codebooks or fine-tuning. \qss No FT uses lattice codebooks but not fine-tuning. \qss uses lattice codebooks and performs fine-tuning.}
\centering
\sc
\begin{tabular}{@{}cccccccc@{}}
\toprule
Model & Method             & Bits & \begin{tabular}[c]{@{}c@{}}ArcC\\ (acc\_norm)\end{tabular} & \begin{tabular}[c]{@{}c@{}}ArcE\\ (acc\_norm)\end{tabular} & \begin{tabular}[c]{@{}c@{}}BoolQ \\ (acc)\end{tabular} & \begin{tabular}[c]{@{}c@{}}PiQA \\ (acc\_norm)\end{tabular} & \begin{tabular}[c]{@{}c@{}}Wino \\ (acc)\end{tabular} \\ \midrule
2-70  & Native             & 16   & 48.0                                                       & 59.7                                                       & 76.6                                                   & 80.9                                                        & 76.8                                                  \\
\rowcolor[HTML]{D9D9D9} 
2-70  & QuIP\# no FT \& no $E_8$ & 4    & \textbf{49.4}                                              & \textbf{60.1}                                              & 77.6                                                   & \textbf{80.7}                                               & 76.1                                                  \\
\rowcolor[HTML]{D9D9D9} 
2-70  & QuIP\# No FT       & 4    & 48.3                                                       & \textbf{60.1}                                              & \textbf{78.4}                                          & 80.6                                                        & 76.2                                                  \\
\rowcolor[HTML]{D9D9D9} 
2-70  & QuIP\#             & 4    & 48.3                                                       & 59.4                                                       & 77.4                                                   & \textbf{80.7}                                               & \textbf{77.1}                                         \\
2-70  & QuIP\# no FT \& no $E_8$ & 3    & 47.4                                                       & 59.1                                                       & 75.8                                                   & \textbf{80.9}                                               & \textbf{77.5}                                         \\
2-70  & QuIP\# No FT       & 3    & 47.9                                                       & \textbf{59.9}                                              & \textbf{78.8}                                          & 79.9                                                        & 77.0                                                  \\
2-70  & QuIP\#             & 3    & \textbf{48.4}                                              & 59.5                                                       & 74.8                                                   & 80.3                                                        & 76.4                                                  \\
\rowcolor[HTML]{D9D9D9} 
2-70  & QuIP\# no FT \& no $E_8$ & 2    & 43.5                                                       & 56.2                                                       & 75.1                                                   & 78.1                                                        & \textbf{76.0}                                         \\
\rowcolor[HTML]{D9D9D9} 
2-70  & QuIP\# No FT       & 2    & 47.2                                                       & \textbf{59.5}                                              & 79.1                                                   & 78.6                                                        & 74.2                                                  \\
\rowcolor[HTML]{D9D9D9} 
2-70  & QuIP\#             & 2    & \textbf{47.7}                                              & 59.1                                                       & \textbf{80.3}                                          & \textbf{79.4}                                               & 75.9                                                  \\
2-13  & Native             & 16   & 44.3                                                       & 58.0                                                       & 69.0                                                   & 79.0                                                        & 69.9                                                  \\
\rowcolor[HTML]{D9D9D9} 
2-13  & QuIP\# no FT \& no $E_8$ & 4    & 43.7                                                       & \textbf{58.6}                                              & \textbf{70.1}                                          & 78.7                                                        & 69.6                                                  \\
\rowcolor[HTML]{D9D9D9} 
2-13  & QuIP\# No FT       & 4    & 42.9                                                       & 56.4                                                       & 67.8                                                   & \textbf{78.9}                                               & \textbf{69.9}                                         \\
\rowcolor[HTML]{D9D9D9} 
2-13  & QuIP\#             & 4    & \textbf{44.2}                                              & 57.7                                                       & 69.7                                                   & \textbf{78.9}                                               & \textbf{69.9}                                         \\
2-13  & QuIP\# no FT \& no $E_8$ & 3    & 42.1                                                       & 55.2                                                       & 70.0                                                   & 77.8                                                        & \textbf{69.5}                                         \\
2-13  & QuIP\# No FT       & 3    & 41.9                                                       & \textbf{57.7}                                              & \textbf{73.3}                                          & 78.1                                                        & 68.0                                                  \\
2-13  & QuIP\#             & 3    & \textbf{43.3}                                              & \textbf{57.7}                                              & 69.8                                                   & \textbf{78.4}                                               & 69.1                                                  \\
\rowcolor[HTML]{D9D9D9} 
2-13  & QuIP\# no FT \& no $E_8$ & 2    & 36.3                                                       & 50.8                                                       & 67.4                                                   & 73.4                                                        & 63.1                                                  \\
\rowcolor[HTML]{D9D9D9} 
2-13  & QuIP\# No FT       & 2    & 37.1                                                       & 50.1                                                       & 66.5                                                   & 75.7                                                        & 63.6                                                  \\
\rowcolor[HTML]{D9D9D9} 
2-13  & QuIP\#             & 2    & \textbf{41.3}                                              & \textbf{55.1}                                              & \textbf{68.3}                                          & \textbf{77.4}                                               & \textbf{67.7}                                         \\
2-7   & Native             & 16   & 40.6                                                       & 53.5                                                       & 71.0                                                   & 76.9                                                        & 67.0                                                  \\
\rowcolor[HTML]{D9D9D9} 
2-7   & QuIP\# no FT \& no $E_8$ & 4    & 39.5                                                       & 51.9                                                       & 71.3                                                   & 76.6                                                        & 67.3                                                  \\
\rowcolor[HTML]{D9D9D9} 
2-7   & QuIP\# No FT       & 4    & \textbf{40.4}                                              & \textbf{53.7}                                              & 68.5                                                   & \textbf{77.2}                                               & 67.5                                                  \\
\rowcolor[HTML]{D9D9D9} 
2-7   & QuIP\#             & 4    & 40.1                                                       & 53.4                                                       & \textbf{69.9}                                          & 76.5                                                        & \textbf{67.6}                                         \\
2-7   & QuIP\# no FT \& no $E_8$ & 3    & 38.1                                                       & 52.6                                                       & 65.2                                                   & 76.1                                                        & 65.1                                                  \\
2-7   & QuIP\# No FT       & 3    & 37.7                                                       & 53.1                                                       & \textbf{70.6}                                          & \textbf{76.7}                                               & \textbf{67.6}                                         \\
2-7   & QuIP\#             & 3    & \textbf{39.4}                                              & \textbf{53.8}                                              & 69.7                                                   & 76.1                                                        & 66.5                                                  \\
\rowcolor[HTML]{D9D9D9} 
2-7   & QuIP\# no FT \& no $E_8$ & 2    & 29.2                                                       & 42.5                                                       & 63.3                                                   & 68.0                                                        & 59.0                                                  \\
\rowcolor[HTML]{D9D9D9} 
2-7   & QuIP\# No FT       & 2    & 32.5                                                       & 42.8                                                       & 62.3                                                   & 71.2                                                        & 62.4                                                  \\
\rowcolor[HTML]{D9D9D9} 
2-7   & QuIP\#             & 2    & \textbf{36.1}                                              & \textbf{50.5}                                              & \textbf{68.3}                                          & \textbf{74.9}                                               & \textbf{64.9}                                         \\ \bottomrule
\end{tabular}
\end{table}

\subsection{More Scaling Plots}

\begin{figure}[H]
\centering
\includegraphics[height=0.4\linewidth]{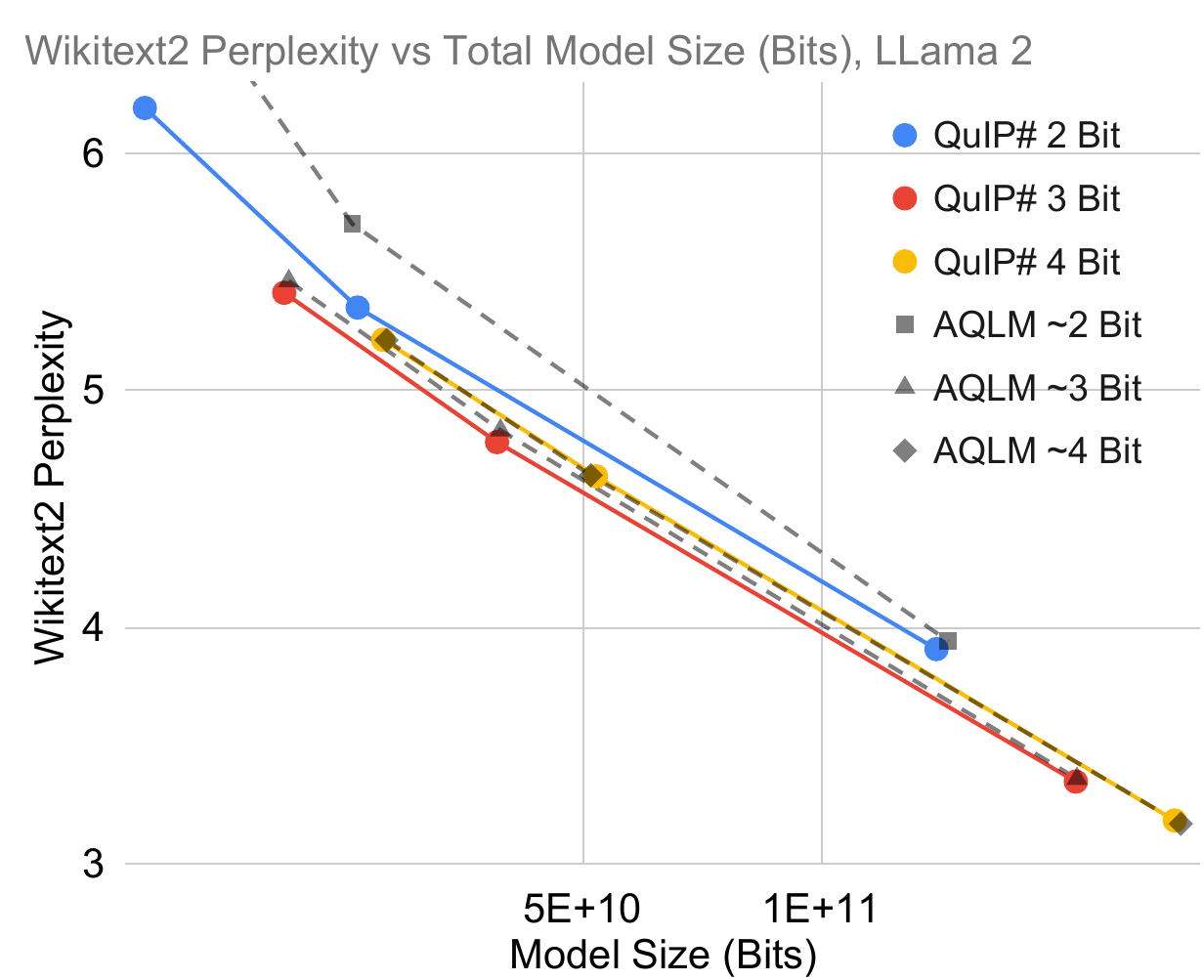}
\includegraphics[height=0.4\linewidth]{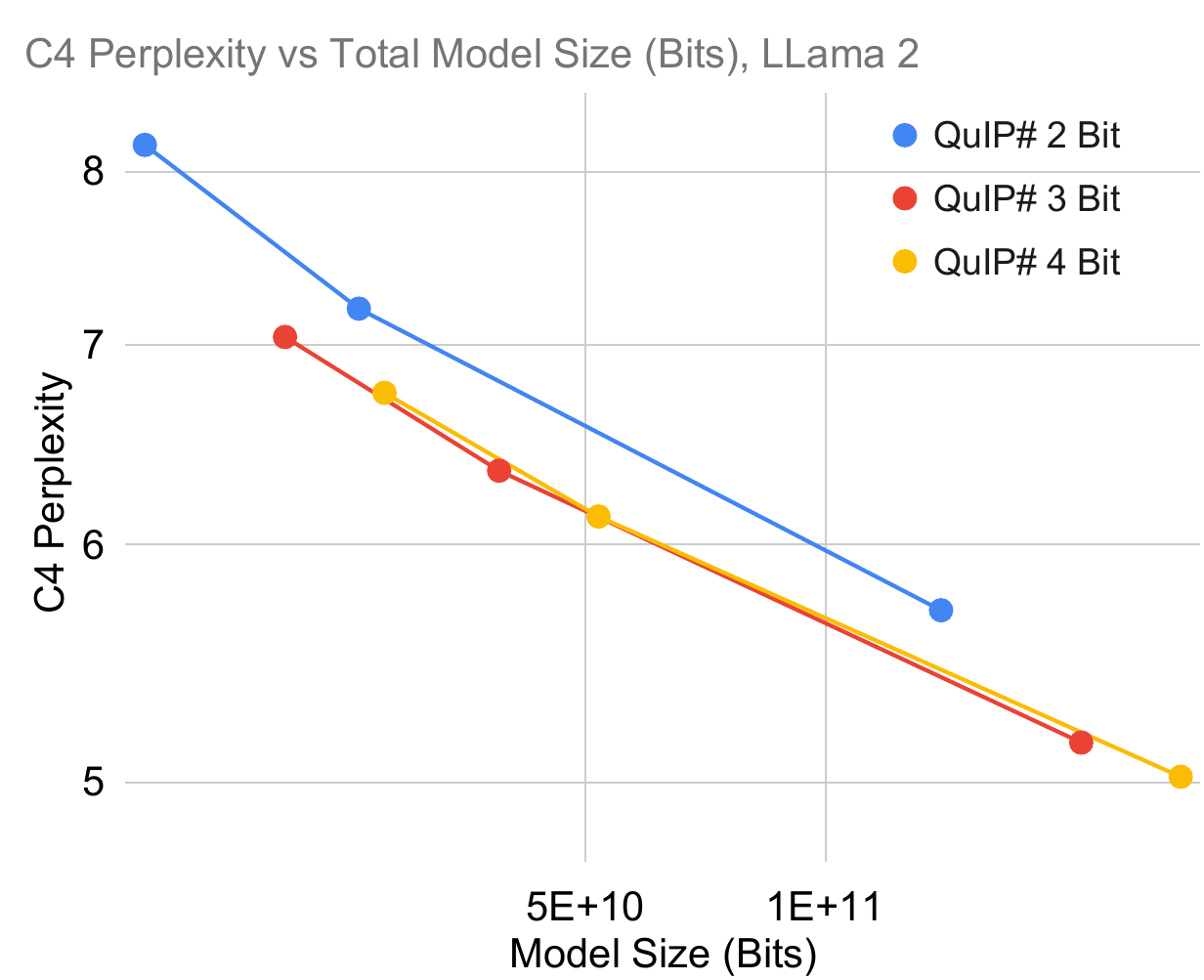}
\includegraphics[height=0.4\linewidth]{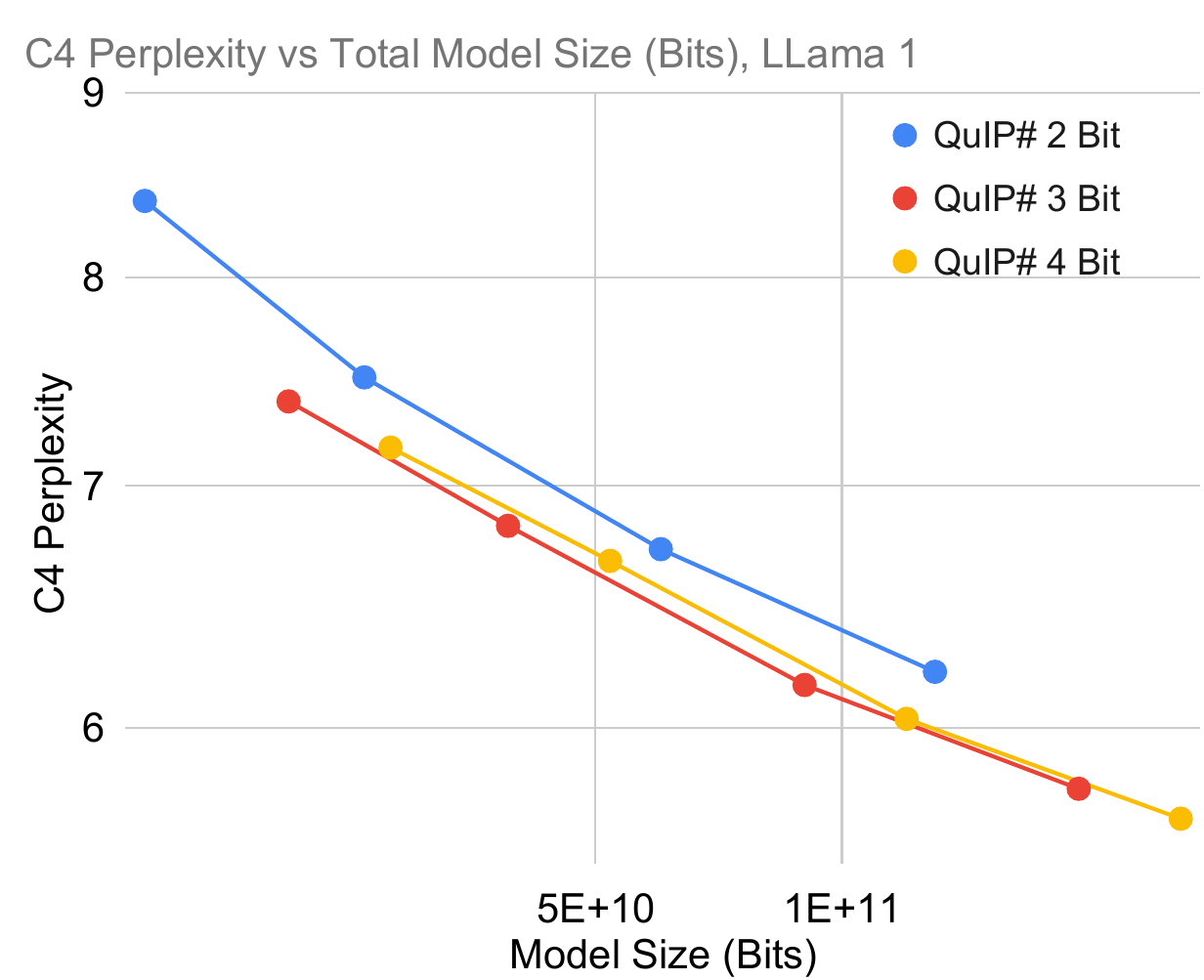}
\caption{\qss scaling. (Top Left) Llama 2 Wikitext 2 perplexity vs AQLM. Context length 4096. \qss 2 and 3 bit scale better than AQLM 2 and 3 bit. (Top Right) Llama 2 C4 Perplexity. Context length 4096. (Bottom) Llama 1 C4 Perplexity. Context length 2048.}
\end{figure}

\section{Implementation Details}

This section contains implementation details for our Llama experiments. 
These details also mostly apply to the Mixtral and Falcon numbers except we use the Falcon dataset \cite{falcon} as it is publicly avaiable.

\subsection{Bit Accounting}
\label{sec:accounting}

The additional overhead of \qss consists of 1KiB for E8P and $16n$ bits for each $n$D sign vector if using fine tuning or $n$ without. 
The 1KiB from E8P is shared over all linear layers, so it adds $\ll 0.01$ bits per weight. 
For a $m\times n$ matrix, the sign vectors take up $\frac{16(n+m)}{nm}$ bits per weight with fine tuning or $\frac{n+m}{nm}$ without.
For LLM-sized matrices (e.g. the smallest matrix in Llama 2 7B is $4096 \times 4096$), this is still $<0.01$ additional bits per weight.

\subsection{Hessian Generation}

Hessian matrices $H$ were generated with 6144 sequences of a model's native context length (2048 for Llama 1, 4096 for Llama 2) from the RedPajama 1T \cite{rpv1} dataset.

\subsection{Hadamard Matrices}

We use Hadamard matrices available at Neil Sloane's website \cite{neilsloane}.

\subsection{Perplexity and Zeroshot Evaluation}

We use the OPTQ \cite{optq} ``Wiktext2'' and ``C4'' (not ``C4 New'') sampling functions to calculate perplexity for our experiments.
We use LM Eval \cite{lmeval} to calculate zeroshot numbers.

\subsection{Scales}

In order to achieve good coverage of the codebook, we scale $W$ by $\rho |W|_F$ before quantizing $W$. For E8P, we used $\rho$ ~= 0.9, for RVQ 3 bit we used $\rho\approx= 0.98$ for the first stage and $\approx2.04$ for the second stage, for RVQ 4 bit we used $\rho\approx= 1.03$ for the first stage and $\approx3.45$ for the second stage. These numbers were determined by finding the $\rho$(s) that minimized the quantization error of quantizing a Gaussian to the codebook. The $\approx$ is because different models have slightly different optimal $\rho$ since incoherence processing does not produce an exact Gaussian. The actual numbers for each model were found with a coarse scale sweep but are all close to these values.

\subsection{Fine Tuning}

For the within-transformer block section of fine-tuning, we use the Adam optimizer \cite{adam}, a learning rate of $5 \times 10^{-5}$, batch size of 8, and sequence length equal to the model's native context length.
We train on small development dataset of 256 sequences from RedPajama 1T and validate on 128 sequences.
We train for 5 epochs (160 steps) and keep the best model parameters based on the validation set.
For the end to end tuning, we use the Adam optimizer, a learning rate of $5 \times 10^{-5}$, batch size of 1, sequence length equal to the model's context length (except for 70B, where we had to use 3072 to avoid an OOM on our not very well optimized training script), and the same dataset and epoch setup as before.
We observe that outside of using a low enough learning rate, the other hyperparameters did not affect fine-tuning much.
For the 2 bit models, we used a learning rate of $5 \times 10^{-4}$ for $S_U$ and $S_V$ ($5 \times 10^{-5}$ for everything else as above) for both the within-block and end to end fine tuning stages.

\subsection{Hardware}
All experiments were run on NVIDIA A100 GPUs except for the timing numbers, which were measured on a NVIDIA RTX 4090 to see what was possible with the current \sota NVIDIA consumer GPU.
We find that we can quantize Llama 2 70B without fine tuning in under 10 GPU-hours and with fine tuning in around 100 GPU-hours.
Both numbers do not include Hessian generation, which can be done once for a model and reused across many different quantization experiments.

\subsection{Code and Prequantized Models}

Our code is available at \url{https://github.com/Cornell-RelaxML/quip-sharp} and prequantized \qss models are available at \url{https://huggingface.co/relaxml}.

\section{Example Generation}

Below are some example generations from Llama 2 70B chat quantized with \qss to 2 bits, truncated to 256 tokens.

\begin{tcolorbox}
\textbf{Prompt:}  How much wood could a woodchuck chuck if a woodchuck could chuck wood?
\vspace{1em}

\textbf{\qs:} 

The answer to this tongue-twister is a bit of a mystery. Woodchucks, also known as groundhogs, are rodents that burrow in the ground and are known for their ability to move large amounts of soil. However, there is no scientific evidence to suggest that they can actually chuck wood.
The phrase ``how much wood could a woodchuck chuck'' was first used in a song called ``The Woodchuck's Song'' in the late 19th century. The song was popularized by the singer and songwriter, Johnny Mercer, and it became a hit in the 1940s. The song's lyrics include the line ``How much wood could a woodchuck chuck if a woodchuck could chuck wood?'' but they don't provide an answer.
In reality, woodchucks don't chuck wood, they dig burrows. They are able to move a lot of soil and dirt as they dig their burrows, but they don't have the ability to move wood. So, the answer to the tongue-
\end{tcolorbox}

\begin{tcolorbox}
\textbf{Prompt:} Write a song about large language models to the tune of Mariah Carey's ``All I Want For Christmas Is You''
\vspace{1em}

\textbf{\qs:}   

    Verse 1:\\
    I don't want a lot for Christmas\\
There's just one thing I'm wishing for\\
A language model that's large and bright\\
To help me with my writing, day and night\\
\\
    Chorus:\\
    All I want for Christmas is a language model\\
A big, smart, AI that's my ideal\\
I'll be happy with just a few billion parameters\\
All I want for Christmas is a language model, yeah\\
\\
    Verse 2:\\
    I'll be good, I won't be naughty\\
I'll use it for good, not foraughty\\
I'll generate text that's nice and neat\\
And make my writing life a treat\\
\\
    Chorus:\\
    All I want for Christmas is a language model\\
A big, smart, AI that's my ideal\\
I'll be happy with just a few billion parameters\\
All I want for Christmas is a language model, yeah\\
\\
    Bridge:\\
    I'\\

\end{tcolorbox}

\begin{tcolorbox}
\textbf{Prompt:} Write some CUDA code that does a matrix multiply
\vspace{1em}

\textbf{\qs:} 

Here is a simple example of a matrix multiply using CUDA:

\begin{verbatim}
#include <cuda.h>
#include <cuda_runtime.h>

__global__ void matrixMultiply(float *A, float *B, float *C, int rowsA, int co
lsA, int rowsB, int colsB) {
  int i = blockIdx.x * blockDim.x + threadIdx.x;
  int j = blockIdx.y * blockDim.y + threadIdx.y;
  if (i < rowsA && j < colsB) {
    float sum = 0;
    for (int k = 0; k < rowsB; k++) {
      sum += A[i + k * colsA] * B[k + j * rowsB];
    }
    C[i + j * rowsB] = sum;
  }
}

int main() {
  int rowsA = 3;
  int colsA = 3;
  int rowsB = 3;
\end{verbatim}
\end{tcolorbox}


\end{document}